\documentclass[final,journal,letterpaper,twocolumn,twoside]{IEEEtran}

\pdfoutput=1

\newcommand{\UTIASmultititle}{Informed Sampling for Asymptotically Optimal Path Planning (Consolidated Version)}
\newcommand{\UTIASsingletitle}{Informed Sampling for Asymptotically Optimal Path Planning (Consolidated Version)}
\newcommand{\UTIASauthor}{\texorpdfstring{%
                                             Gammell \MakeLowercase{\textit{et al.}}%
                                         }{%
                                             Gammell et al.%
                                         }}

\newcommand{\UTIASfooter}{\bf This document consolidates the published paper \cite{gammell_tro18} with its supplementary material and presents it as reviewed.}

\usepackage{amsmath,amssymb,amsfonts,amscd,amsthm}
\usepackage{mathtools} %
\usepackage{fixltx2e} %
\usepackage{microtype} %

\usepackage[numbers]{natbib} %
\usepackage{fancyhdr}
\usepackage{url}

\usepackage[pdftex]{graphicx}
\usepackage{placeins} %

\usepackage[nolist]{acronym} %
\usepackage[usenames,dvipsnames]{color} %
\usepackage[ruled,vlined,linesnumbered]{algorithm2e}  %

\usepackage{fnpos} %
\usepackage{paralist} %
\usepackage{enumitem} %
\setlist{nosep} %

\usepackage{xr} %

\usepackage[pdftex,colorlinks]{hyperref}
\usepackage{hypernat} %
\usepackage{bookmark} %
\usepackage[all]{hypcap} %

\SetAlgoSkip{}

\hypersetup{%
    pdftitle={\UTIASauthor: \UTIASsingletitle},
    pdfauthor={\UTIASauthor},
    pdfkeywords={},
    pdfsubject={},
    pdfstartview=FitH,%
    bookmarksopen=false,%
    breaklinks=true,%
    colorlinks=true,%
    linkcolor=Black,%
    citecolor=Black,%
    urlcolor=Black,%
    anchorcolor=Black,%
    filecolor=Black,%
    menucolor=Black,%
    runcolor=Black %
}%
\theoremstyle{plain}
\newtheorem{lem}{Lemma}
\newtheorem{thm}[lem]{Theorem} %
\newtheorem{cor}[lem]{Corollary} %
\newtheorem{defn}[lem]{Definition} %
\newcommand{\escapeParagraph}{\vspace{1.5ex}}

\newcommand{\revisit}[1]{\tag{\ref{#1} redux}}
\newcommand{\captionbreak}[0]{\\\hspace{\textwidth}}

\newcommand{\algorithmStyle}[0]{\renewcommand{\algorithmcfname}{Alg.}\algorithmCapStart} %

\newcommand{\algo}[1]{Alg.~\ref{#1}}
\newcommand{\algos}[2]{Algs.~\ref{#1}--\ref{#2}}
\newcommand{\algoAnd}[2]{Algs.~\ref{#1} and \ref{#2}}
\newcommand{\algoline}[2]{\algo{#1}, Line~\ref{#1:#2}}
\newcommand{\algolines}[3]{\algo{#1}, Lines~\ref{#1:#2}--\ref{#1:#3}}

\newcommand{\newAlgoLine}[1]{ {\color{BrickRed} #1} }

\let\oldnl\nl%
\newcommand{\skipln}{\renewcommand{\nl}{\let\nl\oldnl}}

\newcommand{\algorithmCapStart}[0]%
{%
    \let\oldhypcapspace\hypcapspace%
    \renewcommand{\hypcapspace}{1.5\baselineskip}%
    \capstart%
    \renewcommand{\hypcapspace}{\let\hypcapspace\oldhypcapspace}%
}

\newcommand{\squeezeWidowedWords}{\looseness=-1}
\newcommand{\resetDisplaySpacing}%
{%
    \setlength \abovedisplayskip{1.5ex plus 4pt minus 2pt}%
    \setlength \abovedisplayshortskip{0pt plus 4pt}%
    \setlength \belowdisplayskip{\abovedisplayskip}%
    \setlength \belowdisplayshortskip{\belowdisplayskip}%
}

\newcommand{\squeezeDisplaySpacing}[1]%
{%
    \squeezeAboveDisplaySpacing{#1}%
    \squeezeBelowDisplaySpacing{#1}%
}
\newcommand{\squeezeAboveDisplaySpacing}[1]%
{%
    \setlength \abovedisplayskip{#1 plus 0.0pt minus 0.0pt}%
}
\newcommand{\squeezeBelowDisplaySpacing}[1]%
{%
    \setlength \belowdisplayskip{#1 plus 0.0pt minus 0.0pt}%
}

\newcommand{\squeezeAboveSpacingForUnderhang}%
{%
    \squeezeAboveDisplaySpacing{0pt}%
}
\newcommand{\squeezeBelowSpacingForUnderhang}%
{%
    \squeezeBelowDisplaySpacing{4pt}%
}

\DeclareMathOperator*{\argmin}{arg\,min}

\DeclareMathOperator*{\diag}{diag}

\newcommand{\bbm}{\begin{bmatrix}}
\newcommand{\ebm}{\end{bmatrix}}

\newcommand{\norm}[2]{\left\| #1 \right\|_{#2}}
\newcommand{\card}[1]{\left|#1\right|}
\newcommand{\suchthat}{\;\; \mbox{s.t.} \;\;}
\newcommand{\expectSymb}[0]{E}
\newcommand{\expectBrackets}[1]{\left[#1\right]}
\newcommand{\expectstBrackets}[2]{\expectBrackets{#1\,\middle|\,#2}}
\newcommand{\expect}[1]{\expectSymb\expectBrackets{#1}}
\newcommand{\expectst}[2]{\expectSymb\expectstBrackets{#1}{#2}}
\newcommand{\prob}[1]{P\left(#1\right)}
\newcommand{\probst}[2]{\prob{#1\;\;\middle|\;\;#2}}
\newcommand{\betaFunc}[2]{\mathrm{B}\left(#1,#2\right)}
\newcommand{\gammaFunc}[1]{\Gamma\left(#1\right)}
\newcommand{\real}[0]{\mathbb{R}}
\newcommand{\positiveReal}[0]{\real_{\geq0}}

\newcommand{\pair}[2]{\left( #1, #2\right)}
\newcommand{\set}[1]{\left\lbrace #1\right\rbrace}
\newcommand{\setst}[2]{\set{#1\;\;\middle|\;\;#2}}

\newcommand{\seqset}[1]{\left(#1\right)}
\newcommand{\seqidx}[4]{\seqset{#1_{#2}}_{#2=#3}^{#4}}
\newcommand{\closure}[1]{\mathrm{cl}\left(#1\right)}

\newcommand{\axis}[0]{\mathbf{a}}
\newcommand{\cost}[0]{c}

\newcommand{\solutionCost}[0]{f}
\newcommand{\costToCome}[0]{g}
\newcommand{\costToGo}[0]{h}
\newcommand{\iter}[0]{i}
\newcommand{\counterj}[0]{j}
\newcommand{\counterk}[0]{k}
\newcommand{\mapwidth}[0]{l}

\newcommand{\dimension}[0]{n}
\newcommand{\pdfSymb}[0]{p}
\newcommand{\totalSamples}[0]{q}
\newcommand{\numGoals}[0]{z}
\newcommand{\radius}[0]{r}
\newcommand{\pathCostSymb}[0]{\cost}

\newcommand{\stateu}[0]{\mathbf{u}}
\newcommand{\statev}[0]{\mathbf{v}}
\newcommand{\obswidth}[0]{w}
\newcommand{\statew}[0]{\mathbf{w}}
\newcommand{\statex}[0]{\mathbf{x}}
\newcommand{\statey}[0]{\mathbf{y}}

\newcommand{\ellipseMatrix}[0]{\mathbf{S}}
\newcommand{\ellipseRotation}[0]{\mathbf{C}_{\rm we}}

\newcommand{\iterThresh}[0]{\kappa}
\newcommand{\rrewire}[0]{\radius_{\mathrm{rewire}}}

\newcommand{\rrrtstar}[0]{\radius_{\mathrm{RRT}^*}}
\newcommand{\rrrtstarmin}[0]{\radius_{\mathrm{RRT}^*}^*}

\newcommand{\krrtstar}[0]{k_{\mathrm{RRT}^*}}
\newcommand{\krrtstarmin}[0]{k_{\mathrm{RRT}^*}^*}
\newcommand{\maxEdge}[0]{\eta}
\newcommand{\lebesgueSymb}[0]{\lambda}
\newcommand{\rate}[0]{\mu}

\newcommand{\pathSeq}[0]{\sigma}

\newcommand{\linTransSymb}[0]{\tau}
\newcommand{\unitBall}[1]{\zeta_{#1}}

\newcommand{\uniformSymb}[0]{\mathcal{U}}
\newcommand{\uniform}[1]{\uniformSymb\left(#1\right)}
\newcommand{\setInsert}[0]{\xleftarrow{\scriptscriptstyle +}}
\newcommand{\setRemove}[0]{\xleftarrow{\scriptscriptstyle -}}
\newcommand{\lebesgue}[1]{\lebesgueSymb\left(#1\right)}
\newcommand{\limitoinf}[0]{\lim_{\iter\to\infty}}

\newcommand{\invLinTransSymb}[0]{\linTransSymb^{-1}}
\newcommand{\linTrans}[1]{\linTransSymb\left(#1\right)}
\newcommand{\invLinTrans}[1]{\invLinTransSymb\left(#1\right)}

\newcommand{\stateSet}[0]{X}

\newcommand{\vertexSet}[0]{V}
\newcommand{\edgeSet}[0]{E}
\newcommand{\edgeCost}[0]{\cost}

\newcommand{\treeGraph}[0]{\mathcal{T}}

\newcommand{\pathSet}[0]{\Sigma}
\newcommand{\bestPath}[0]{\pathSeq^{*}}

\newcommand{\samplePath}[0]{\pathSeq_\totalSamples}

\newcommand{\pathCost}[1]{\pathCostSymb\left(#1\right)}

\newcommand{\costFromBelow}[1]{\widehat{#1}}
\newcommand{\trueCost}[1]{#1}
\newcommand{\costFromAbove}[1]{#1_{\treeGraph}}

\newcommand{\gTrue}[1]{\trueCost{\costToCome}\left(#1\right)}
\newcommand{\gAbove}[1]{\costFromAbove{\costToCome}\left(#1\right)}

\newcommand{\hTrue}[1]{\trueCost{\costToGo}\left(#1\right)}

\newcommand{\fBelow}[1]{\costFromBelow{\solutionCost}\left(#1\right)}
\newcommand{\fTrue}[1]{\trueCost{\solutionCost}\left(#1\right)}

\newcommand{\cTrue}[2]{\trueCost{\edgeCost}\left(#1,#2\right)}

\newcommand{\namedLebesgue}[1]{\lebesgueSymb_{#1}}
\newcommand{\phsMeasure}[0]{\namedLebesgue{\rm PHS}}

\newcommand{\namedSet}[1]{\stateSet_{#1}}
\newcommand{\fTrueSet}[0]{\namedSet{\trueCost{\solutionCost}}}
\newcommand{\fBelowSet}[0]{\namedSet{\costFromBelow{\solutionCost}}}
\newcommand{\fBelowSetj}[0]{\namedSet{\costFromBelow{\solutionCost},\counterj}}
\newcommand{\fBelowSetk}[0]{\namedSet{\costFromBelow{\solutionCost},\counterk}}
\newcommand{\obsSet}[0]{\namedSet{\rm obs}}
\newcommand{\freeSet}[0]{\namedSet{\rm free}}
\newcommand{\goalSet}[0]{\namedSet{\rm goal}}
\newcommand{\samplingSet}[0]{\namedSet{\rm samp}}

\newcommand{\ballSet}[0]{\namedSet{\rm ball}}
\newcommand{\ellipseSet}[0]{\namedSet{\rm ellipse}}
\newcommand{\rectSet}[0]{\namedSet{\rm rect}}
\newcommand{\phsSet}[0]{\namedSet{\rm PHS}}
\newcommand{\phsSetj}[0]{\namedSet{{\rm PHS},\counterj}}
\newcommand{\phsSetk}[0]{\namedSet{{\rm PHS},\counterk}}

\newcommand{\namedVertexSet}[1]{\vertexSet_{#1}}
\newcommand{\pruneVertices}[0]{\namedVertexSet{\rm prune}}
\newcommand{\solnVertices}[0]{\namedVertexSet{\mbox{\rm\scriptsize sol'n}}}

\newcommand{\nearVertices}[0]{\namedVertexSet{\rm near}}

\newcommand{\namedState}[1]{\statex_{#1}}

\newcommand{\xstart}[0]{\namedState{\rm start}}
\newcommand{\xgoal}[0]{\namedState{\rm goal}}
\newcommand{\xgoalj}[0]{\namedState{{\rm goal},\counterj}}
\newcommand{\xgoalk}[0]{\namedState{{\rm goal},\counterk}}

\newcommand{\xball}[0]{\namedState{\rm ball}}
\newcommand{\xellipse}[0]{\namedState{\rm ellipse}}
\newcommand{\xcentre}[0]{\namedState{\rm centre}}
\newcommand{\xrand}[0]{\namedState{\rm rand}}
\newcommand{\xnew}[0]{\namedState{\rm new}}
\newcommand{\namedVertex}[1]{\statev_{#1}}
\newcommand{\vgoal}[0]{\namedVertex{\rm goal}}

\newcommand{\vmin}[0]{\namedVertex{\rm min}}
\newcommand{\vnear}[0]{\namedVertex{\rm near}}
\newcommand{\vnearest}[0]{\namedVertex{\rm nearest}}
\newcommand{\vparent}[0]{\namedVertex{\rm parent}}

\newcommand{\namedCost}[1]{\cost_{#1}}

\newcommand{\ccur}[0]{\namedCost{\iter}}
\newcommand{\cprev}[0]{\namedCost{\iter-1}}
\newcommand{\cnext}[0]{\namedCost{\iter+1}}
\newcommand{\cnew}[0]{\namedCost{\rm new}}
\newcommand{\cmin}[0]{\namedCost{\rm min}}
\newcommand{\cnear}[0]{\namedCost{\rm near}}

\newcommand{\copt}[0]{\cost^*}

\newcommand{\namedRate}[1]{\rate_{#1}}
\newcommand{\rrtstarRate}[0]{\namedRate{\mathrm{RRT}^*}}
\newcommand{\rejectRate}[0]{\namedRate{\mathrm{Rect}}}
\newcommand{\informedRate}[0]{\namedRate{\mathrm{Inf}}}

\newcommand{\namedPdfSymb}[1]{\pdfSymb_{#1}}
\newcommand{\namedPdf}[2]{\namedPdfSymb{#1}\left(#2\right)}
\newcommand{\ballPdf}[1]{\namedPdf{\rm ball}{#1}}
\newcommand{\ellipsePdf}[1]{\namedPdf{\rm ellipse}{#1}}
\newcommand{\probBetter}[0]{\namedPdfSymb{\solutionCost}}

\newcommand{\RRTsharp}{\acs{RRT}\textsuperscript{\#}}
\newcommand{\RRTx}{\acs{RRT}\textsuperscript{X}}

\begin{document}

\begin{acronym}[UTIAS]

    \acro{ASRL}{Autonomous Space Robotics Lab}
    \acro{CMU}{Carnegie Mellon University}
    \acro{CSA}{Canadian Space Agency}
    \acro{DRDC}{Defence Research and Development Canada}
    \acro{KSR}{Koffler Scientific Reserve at Jokers Hill}
    \acro{MET}{Mars Emulation Terrain}
    \acro{MIT}{Massachusetts Institute of Technology}
    \acro{NASA}{National Aeronautics and Space Administration}
    \acro{NSERC}{Natural Sciences and Engineering Research Council of Canada}
    \acro{NCFRN}{\acs{NSERC} Canadian Field Robotics Network}
    \acro{NORCAT}{Northern Centre for Advanced Technology Inc.}
    \acro{ODG}{Ontario Drive and Gear Ltd.}
    \acro{ONR}{Office of Naval Research}
    \acro{USSR}{Union of Soviet Socialist Republics}
    \acro{UofT}{University of Toronto}
    \acro{UW}{University of Waterloo}
    \acro{UTIAS}{University of Toronto Institute for Aerospace Studies}

    \acro{ACPI}{advanced configuration and power interface}
    \acro{CLI}{command-line interface}
    \acro{GUI}{graphical user interface}
    \acro{JIT}{just-in-time}
    \acro{LAN}{local area network}
    \acro{MFC}{Microsoft foundation class}
    \acro{NIC}{network interface card}
    \acro{SDK}{software development kit}
    \acro{HDD}{hard-disk drive}
    \acro{SSD}{solid-state drive}

    \acro{ICRA}{IEEE International Conference on Robotics and Automation}
    \acro{IJRR}{International Journal of Robotics Research}
    \acro{IROS}{IEEE/RSJ International Conference on Intelligent Robots and Systems}
    \acro{RSS}{Robotics: Science and Systems Conference}
    \acro{TRO}[T-RO]{IEEE Transactions on Robotics}

    \acro{DOF}{degree-of-freedom}
        \acrodefplural{DOF}[DOFs]{degrees-of-freedom} %

        \acro{FOV}{field of view}
            \acrodefplural{FOV}[FOVs]{fields of view}
        \acro{HDOP}{horizontal dilution of position}
        \acro{UTM}{universal transverse mercator}
        \acro{WAAS}{wide area augmentation system}
        \acro{AHRS}{attitude heading reference system}
        \acro{DAQ}{data acquisition}
        \acro{DGPS}{differential global positioning system}
        \acro{DPDT}{double-pole, double-throw}
        \acro{DPST}{double-pole, single-throw}
        \acro{GPR}{ground penetrating radar}
        \acro{GPS}{global positioning system}
        \acro{LED}{light-emitting diode}
        \acro{IMU}{inertial measurement system}
        \acro{PTU}{pan-tilt unit}
        \acro{RTK}{real-time kinematic}
        \acro{R/C}{radio control}
        \acro{SCADA}{supervisory control and data acquisition}
        \acro{SPST}{single-pole, single-throw}
        \acro{SPDT}{single-pole, double-throw}
        \acro{UWB}{ultra-wide band}

    \acro{DDS}{Departmental Doctoral Seminar}
    \acro{DEC}{Doctoral Examination Committee}
    \acro{FOE}{Final Oral Exam}
    \acro{ICD}{Interface Control Document}

    \acro{iid}[i.i.d.]{independent and identically distributed}
    \acro{aas}[a.a.s.]{asymptotically almost-surely}
    \acro{RGG}{random geometric graph}
    \acro{BVP}{boundary-value problem}
    \acro{2BVP}[2-pt BVP]{two-point \acl{BVP}}

    \acro{EKF}{extended Kalman filter}
    \acro{iSAM}{incremental smoothing and mapping}
    \acro{ISRU}{in-situ resource utilization}
    \acro{PCA}{principle component analysis}
    \acro{SLAM}{simultaneous localization and mapping}
    \acro{SVD}{singular value decomposition}
    \acro{UKF}{unscented Kalman filter}
    \acro{VO}{visual odometry}
    \acro{VTR}[VT\&R]{visual teach and repeat}

        \acro{ADstar}[AD*]{Anytime D∗}
        \acro{ADAstar}[ADA*]{Anytime Dynamic A*}
        \acro{ARAstar}[ARA*]{Anytime Repairing A*}
        \acro{BITstar}[BIT*]{Batch Informed Trees}
            \acrodefplural{BITstar}[BIT*]{Batch Informed Trees}
        \acro{RABITstar}[RABIT*]{Regionally Accelerated \acs{BITstar}}
        \acro{BRM}{belief roadmap}
        \acro{CFOREST}[C-FOREST]{Coupled Forest of Random Engrafting Search Trees}
        \acro{CHOMP}{Covariant Hamiltonian Optimization for Motion Planning}
        \acro{Dstar}[D*]{Dynamic A∗}
        \acro{EET}{Exploring/Exploiting Tree} %
        \acro{EST}{Expansive Space Tree}
            \acrodefplural{EST}[EST]{Expansive Space Trees}
        \acro{FMTstar}[FMT*]{Fast Marching Tree}
            \acrodefplural{FMTstar}[FMT*]{Fast Marching Trees}
        \acro{GVG}{generalized Voronoi graph}
        \acro{hRRT}{Heuristically Guided \acs{RRT}}
        \acro{LBTRRT}[LBT-RRT]{Lower Bound Tree \acs{RRT}}
        \acro{LQG-MP}{linear-quadratic Gaussian motion planning}
        \acro{LPAstar}[LPA*]{Lifelong Planning A*}
        \acro{MPLB}{Motion Planning Using Lower Bounds}
        \acro{MDP}{Markov decision process}
            \acrodefplural{MDP}[MDPs]{Markov decision processes}
        \acro{NRP}{network of reusable paths}
            \acrodefplural{NRP}[NRPs]{networks of reusable paths}
        \acro{POMDP}{partially-observable Markov decision process}
            \acrodefplural{POMDP}[POMDPs]{partially-observable Markov decision processes}
        \acro{PRM}{Probabilistic Roadmap}
            \acrodefplural{PRM}[PRM]{Probabilistic Roadmaps}
        \acro{PRMstar}[PRM*]{asymptotically optimal \acs{PRM}}
            \acrodefplural{PRMstar}[PRM*]{asymptotically optimal \acsp{PRM}}
        \acro{RAstar}[RA*]{Randomized A*}
        \acro{RRBT}{rapidly-exploring random belief tree}
        \acro{RRG}{Rapidly exploring Random Graph}
            \acrodefplural{RRG}[RRG]{Rapidly exploring Random Graphs}
        \acro{RRM}{Rapidly exploring Roadmap}
        \acro{RRT}{Rapidly exploring Random Tree}
            \acrodefplural{RRT}[RRT]{Rapidly-exploring Random Trees}
        \acro{RRTstar}[RRT*]{asymptotically optimal \acs{RRT}}
            \acrodefplural{RRTstar}[RRT*]{asymptotically optimal \acsp{RRT}}
        \acro{SBAstar}[SBA*]{Sampling-based A*}
        \acro{sPRM}[s-PRM]{simplified \acs{PRM}}
            \acrodefplural{s-PRM}[sPRM]{simplified \acsp{PRM}}
        \acro{STOMP}{Stochastic Trajectory Optimization for Motion Planning}
        \acro{TRRT}[T-RRT]{Transition-based \acs{RRT}}
            \acro{TRRTstar}[T-RRT*]{Transition-based \acs{RRTstar}}
            \acro{ATRRT}[AT-RRT]{Anytime Transition-based \acs{RRT}}

    \acro{HERB}{Home Exploring Robot Butler}
    \acro{MER}{Mars Exploration Rover}
    \acro{MSL}{Mars Science Laboratory}
    \acro{OMPL}{Open Motion Planning Library}
    \acro{ROS}{Robot Operating System}

\end{acronym}

\title{\UTIASmultititle}

\author
{%
    Jonathan~D.~Gammell,~\IEEEmembership{Member,~IEEE},\\
    Timothy~D.~Barfoot,~\IEEEmembership{Senior Member,~IEEE},
    and~Siddhartha~S.~Srinivasa,~\IEEEmembership{Senior Member,~IEEE}%
    \thanks{J.\ D.\ Gammell performed this work as a member of the Autonomous Space Robotics Lab at the University of Toronto Institute for Aerospace Studies. He is now with the Oxford Robotics Institute at the University of Oxford, Oxford, United Kingdom. Email: \texttt{gammell@robots.ox.ac.uk}}
    \thanks{T.\ D.\ Barfoot is with the Autonomous Space Robotics Lab at the University of Toronto Institute for Aerospace Studies, Toronto, Ontario, Canada. Email: \texttt{tim.barfoot@utoronto.ca}}%
    \thanks{S.\ S.\ Srinivasa is with the School of Computer Science and Engineering, University of Washington, Seattle, Washington, USA. Email: \texttt{siddh@cs.uw.edu}}%
    \thanks{Manuscript submitted June 20, 2017.}%
}

\markboth{\MakeUppercase{IEEE Transactions on Robotics}}%
{\MakeUppercase{\UTIASauthor: \UTIASsingletitle}}

\maketitle
\fancyhf{}
\pagestyle{fancy}
\fancyfoot[C]{\UTIASfooter}
\renewcommand\headrulewidth{0pt}
\fancypagestyle{IEEEtitlepagestyle}{ %
  \fancyfoot[C]{\UTIASfooter}%
  \renewcommand{\headrulewidth}{0pt} %
}

\begin{abstract}
Anytime almost-surely asymptotically optimal planners, such as \acs{RRTstar}\acused{RRTstar}, incrementally find paths to every state in the search domain.
This is inefficient once an initial solution is found as then only states that can provide a \emph{better} solution need to be considered.
Exact knowledge of these states requires solving the problem but can be approximated with heuristics.

This paper formally defines these sets of states and demonstrates how they can be used to analyze arbitrary planning problems.
It uses the well-known $L^2$ norm (i.e., Euclidean distance) to analyze minimum-path-length problems and shows that existing approaches decrease in effectiveness \emph{factorially} (i.e., faster than exponentially) with state dimension.
It presents a method to address this curse of dimensionality by \emph{directly} sampling the prolate hyperspheroids (i.e., symmetric $\dimension$-dimensional ellipses) that define the $L^2$ \emph{informed} set.

The importance of this direct informed sampling technique is demonstrated with Informed \acs{RRTstar}.
This extension of \acs{RRTstar} has less theoretical dependence on state dimension and problem size than existing techniques and allows for \emph{linear} convergence on some problems.
It is shown experimentally to find better solutions faster than existing techniques on both abstract planning problems and \acs{HERB}, a two-arm manipulation robot.
\end{abstract}
\acresetall
\acused{RRTstar}
\acused{PRMstar}


\section{Introduction}
\IEEEPARstart{T}{here} are many powerful path planning techniques in robotics.
Popular approaches include graph-based searches, such as Dijkstra's algorithm \cite{dijkstra_59} and A* \cite{hart_tssc68}, and sampling-based methods, such as \acp{PRM} \cite{kavraki_tro96}, \acp{EST} \cite{hsu_ijrr02}, and \acp{RRT} \cite{lavalle_ijrr01}.
While sampling-based methods avoid the challenges of \emph{a priori} discretizations, their stochastic nature limits their formal performance.
They are said to be \emph{probabilistically complete} if the probability of finding a solution, if one exists, approaches unity with an infinite number of samples.
They are also said to be \emph{almost-surely asymptotically optimal} if the probability of converging asymptotically to the optimum, if one exists, approaches unity with an infinite number of samples (e.g., \ac{RRTstar} \cite{karaman_ijrr11}).

\begin{figure}[tb]
    \centering
    \includegraphics[width=\columnwidth]{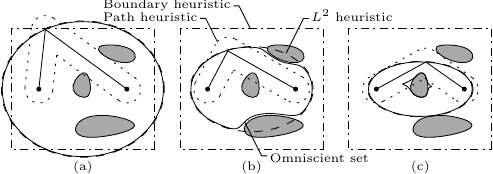}
    \caption{%
    An illustration of how the set of states that can improve solution length shrinks as better solutions are found.
    Common estimates of this \emph{omniscient set} are illustrated as \textit{informed sets}.
    The $L^2$ informed set always contains the entire omniscient set (i.e., $100\%$ recall) and shrinks along with it as a function of the current solution (i.e., high precision).
    It is exactly equal to the omniscient set in the absence of obstacles and constraints (i.e., $100\%$ recall \emph{and} precision).
    This paper shows that direct sampling this $L^2$ informed set is a necessary condition for almost-surely asymptotically optimal planners to scale effectively to high state dimensions.
    This technique is demonstrated with Informed \acs{RRTstar}.
    }
    \label{fig:intro:infsets}
\end{figure}%

\acs{RRT} searches a planning problem by incrementally building a tree through free space.
\ac{RRTstar} extends this procedure to incrementally rewire the tree during its construction.
This rewiring locally optimizes \emph{every} vertex in the tree and allows the algorithm to almost-surely converge asymptotically to the optimal path to \emph{every state} in the problem domain.
This is an inefficient way to find the optimal solution to a single planning query.
\squeezeWidowedWords

The only states that need to be considered in single-query scenarios are those that can provide a better solution \cite{ferguson_iros06}.
While exact knowledge of these states requires solving the planning problem, they can often be approximated with heuristics (Fig.~\ref{fig:intro:infsets}).
These heuristics have previously been used to focus almost-surely asymptotically optimal search \cite{akgun_iros11,otte_tro13} but can also provide insight into the optimal planning problem.

This paper uses the set of states that can provide a better solution to analyze incremental almost-surely asymptotically optimal planning.
It formally defines this shrinking set as the \emph{omniscient set} and shows that sampling it is a necessary condition for \acs{RRTstar}-style planners to improve a solution.
It defines estimates of this set as \emph{informed sets} and provides metrics to quantify them in terms of their compactness (i.e., \emph{precision}) and completeness (i.e., \emph{recall}).
It uses these results to bound the probability of improving a solution to a holonomic planning problem by the probability of sampling an informed set with $100\%$ recall.

The $L^2$ norm (i.e., Euclidean distance) is a well-known heuristic for problems seeking to minimize path length.
It describes the omniscient set exactly in the absence of obstacles\pagebreak

\noindent and constraints (i.e., it is \textit{sharp}\footnotemark{}) and always contains the omniscient set of a problem (i.e., it is \emph{universally admissible}).
This paper uses it to analyze the minimum-path-length problem and shows that existing focusing techniques (e.g., \cite{akgun_iros11,otte_tro13}) are ineffective in high state dimensions.
It is proven that these rejection-sampling approaches have a probability of improving a solution that goes to zero \emph{factorially} (i.e., faster than exponentially) as state dimension increases.%
\footnotetext{A bound or estimate is said to be sharp if it is exactly equal to the true value (i.e., has perfect precision and recall) in at least one case.}

This paper demonstrates how this minimum-path-length \emph{curse of dimensionality} can be reduced by directly sampling the symmetric $\dimension$-dimensional ellipse (i.e., prolate hyperspheroid), the $L^2$ informed set.
The presented direct sampling approach always finds states that are believed to belong to a better solution regardless of the relative size of the $L^2$ informed set.
It outperforms existing focusing techniques by orders of magnitude as state dimension increases.

The informed search approach is demonstrated with Informed \acs{RRTstar}.
This extension of \ac{RRTstar} uses direct informed sampling and admissible graph pruning to focus the search for improvements.
It is shown analytically to outperform existing techniques in terms of convergence rate, especially in high state dimensions, and to result in linear convergence on some problems.
It is probabilistically complete and almost-surely asymptotically optimal.
When the $L^2$ heuristic does not provide additional information (e.g., small planning problems and/or large informed sets) it is identical to \ac{RRTstar}.
A version of Informed \ac{RRTstar} is publicly available in the \acf{OMPL} \cite{ompl}.

Informed \acs{RRTstar} is evaluated experimentally on abstract problems and on the \acs{CMU} Personal Robotic Lab's \ac{HERB}  \cite{herb}, a 14-\ac{DOF} mobile manipulation platform.
These experiments show that it outperforms existing focusing techniques as state dimension increases, especially in problems with large planning domains.
\squeezeWidowedWords

This paper is organized as follows.
Section~\ref{sec:omni_inf} defines omniscient and informed sets and their associated precision and recall in preparation for the literature review presented in Section~\ref{sec:lit}.
Section~\ref{sec:l2} presents a direct informed sampling technique for problems seeking to minimize path length which is demonstrated with Informed \ac{RRTstar} in Section~\ref{sec:inf}.
Section~\ref{sec:rate} analyzes the expected convergence rate of \ac{RRTstar} algorithms and Section~\ref{sec:exp} demonstrates the practical advantages of this improvement on abstract and simulated problems.
Section~\ref{sec:fin} finally presents a closing discussion and thoughts on future work.
\squeezeWidowedWords

\subsection{Statement of Contributions}
This paper is a continuation of ideas first published in \cite{gammell_iros14} and associated technical reports \cite{gammell_arxiv14,gammell_arxiv14b} and makes the following specific contributions:
\begin{itemize}
    \item Formally defines omniscient and informed sets (Definitions~\ref{defn:omni} and \ref{defn:informed}) and demonstrates how precision and recall can be used to quantify the performance of informed sampling (Definitions~\ref{defn:precision} and \ref{defn:recall}).
    \item Provides upper bounds on the probability that an incremental sampling-based planner improves a solution to a holonomic planning problem (Theorems~\ref{thm:necessary:exact:prob} and \ref{thm:necessary:heuristic:prob}).
    \item Bounds the expected next-iteration cost for \ac{RRTstar} algorithms on any minimum-path-length planning problem (Lemma~\ref{lem:conv:expect}) and shows that existing formulations of these algorithms for holonomic planning have a probability of improving a solution that decreases factorially with state dimension (Theorem~\ref{thm:curse}).
    \item Develops a method to reduce this minimum-path-length curse of dimensionality by directly sampling the ellipsoidal $L^2$ informed set defined by a goal or \emph{countable set of goals} and the current solution (\algos{algo:infset}{algo:randomKeep}).
    \item Proves that a planning algorithm using this approach, Informed \acs{RRTstar}, has better theoretical convergence (Theorems~\ref{thm:conv:rrtstar}--\ref{thm:conv:informed}) and experimental performance than existing focused planning algorithms on holonomic problems.
\end{itemize}

\section{Omniscient and Informed Sets}\label{sec:omni_inf}
A formal discussion of the optimal planning problem is presented in support of the literature review.
It includes definitions of the states that can provide a better solution, the \emph{omniscient set} (Definition~\ref{defn:omni}), and estimates of this set, \emph{informed sets}, quantified by \emph{precision} and \emph{recall} (Definitions~\ref{defn:informed}--\ref{defn:admissible}).
These sets provide theoretical upper bounds on the probability of improving a solution to a holonomic problem that are used throughout the remainder of the paper (Theorems~\ref{thm:necessary:exact:prob} and \ref{thm:necessary:heuristic:prob}).

Finding the optimal path from a start to a goal is formally defined as the optimal planning problem  (Definition~\ref{defn:opt}).
The given definition is similar to \cite{karaman_ijrr11}.
\begin{defn}[Optimal planning]\label{defn:opt}
    Let $\stateSet \subseteq \real^\dimension$ be the state space of the planning problem, $\obsSet \subset \stateSet$ be the states in collision with obstacles, and $\freeSet = \closure{\stateSet \setminus \obsSet}$ be the resulting set of permissible states, where $\closure{\cdot}$ represents the closure of a set.
    Let $\xstart \in \freeSet$ be the initial state and $\goalSet \subset \freeSet$ be the set of desired goal states.
    Let $\pathSeq : \; \left[0,1\right] \to \freeSet$ be a sequence of states through collision-free space that can be executed by the robot (i.e., a collision-free feasible path) and $\pathSet$ be the set of all such nontrivial paths.

    The optimal planning problem is then formally defined as the search for a path, $\bestPath\in\pathSet$, that minimizes a given cost function, $\pathCostSymb : \; \pathSet \to \positiveReal$, while connecting $\xstart$ to $\xgoal\in\goalSet$,
    \begin{equation*}
        \bestPath = \argmin_{\pathSeq \in \pathSet} \setst{\pathCost{\pathSeq}}{\pathSeq\left(0\right) = \xstart,\, \pathSeq\left(1\right) \in \goalSet},
    \end{equation*}%
    where $\positiveReal$ is the set of non-negative real numbers.
\end{defn}

Many sampling-based planners, such as \ac{RRTstar}, probabilistically converge towards the optimum of these problems.
Such planners are described as probabilistically complete and almost-surely asymptotically optimal (Definition~\ref{defn:asao}).
\begin{defn}[Almost-sure asymptotic optimality]\label{defn:asao}
    A planner is said to be almost-surely asymptotically optimal if, with an infinite number of samples, the probability of converging asymptotically to the optimum (Definition~\ref{defn:opt}), if one exists, is one,\squeezeWidowedWords
    \begin{equation*}
        \prob{\limsup_{\totalSamples \to \infty}\pathCost{\samplePath} = \pathCost{\bestPath}} = 1,
    \end{equation*}
    where $\totalSamples$ is the number of samples, $\samplePath$ is the path found by the planner from those samples, $\bestPath$ is the optimal solution to the planning problem, and $\pathCost{\cdot}$ is the cost of a path.
\end{defn}

Once \emph{any} solution is found, the set of states that can provide a \emph{better} solution can be defined as the omniscient set (Definition~\ref{defn:omni}).
\begin{defn}[Omniscient set]\label{defn:omni}
    Let $\gTrue{\statex}$ be the cost of the optimal path from the start to a state, $\statex\in\freeSet$, the \textit{optimal cost-to-come},
    \begin{equation*}
        \gTrue{\statex} \coloneqq \min_{\pathSeq \in \pathSet} \setst{\pathCost{\pathSeq}}{\pathSeq(0) = \xstart,\, \pathSeq(1) = \statex},
    \end{equation*}%
    and $\hTrue{\statex}$ be the cost of the optimal path from $\statex$ to the goal region, the \textit{optimal cost-to-go},
    \begin{equation*}
        \hTrue{\statex} \coloneqq \min_{\pathSeq \in \pathSet} \setst{\pathCost{\pathSeq}}{\pathSeq(0) = \statex,\, \pathSeq(1) \in \goalSet}.
    \end{equation*}%
    The cost of the optimal path from $\xstart$ to $\goalSet$ constrained to pass through $\statex$ is then given by $\fTrue{\statex} \coloneqq \gTrue{\statex} + \hTrue{\statex}$.
    This defines the subset of states that can belong to a solution better than the current solution, $\ccur$, as
    \begin{equation}\label{eqn:fset}
        \fTrueSet \coloneqq \setst{\statex \in \freeSet}{\fTrue{\statex} < \ccur}.
    \end{equation}%
    Exact knowledge of $\fTrueSet$ requires exact knowledge of the entire planning problem so we  refer to it as the \emph{omniscient set}.
\end{defn}

\begin{figure}[tb]
    \centering
    \includegraphics[width=\columnwidth]{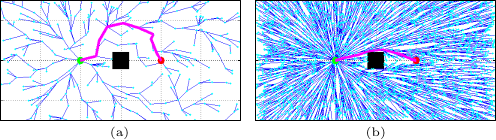}
    \caption
    {%
        An example of how \acs{RRTstar} almost-surely converges asymptotically to the optimum by incrementally building and rewiring a tree through the \emph{entire} problem domain.
        \ac{RRTstar} incrementally expands the tree into the problem domain and improve its connections.
        By continuing this process indefinitely, it almost-surely converges asymptotically to the optimal solution by asymptotically improving every path in the tree.
        This is inefficient in single-query planning scenarios.
        \squeezeWidowedWords
    }
    \label{fig:rrtstar:all}
\end{figure}
\acs{RRTstar} builds a tree by incrementally adding states from the problem domain (Fig.~\ref{fig:rrtstar:all}).
A necessary condition for it to improve a solution is that the newly added state belongs to the omniscient set (Lemma~\ref{lem:necessary:exact}).

\begin{lem}[The necessity of adding states in the omniscient set]\label{lem:necessary:exact}
    Adding a state from the omniscient set, $\xnew\in\fTrueSet$, is a necessary condition for \ac{RRTstar} to improve the current solution, $\ccur$,\squeezeWidowedWords
    \begin{equation*}
        \cnext < \ccur \implies \xnew \in \fTrueSet.
    \end{equation*}
    
    This condition is necessary but not \emph{sufficient} to improve the solution as the ability of states in $\fTrueSet$ to provide better solutions at any iteration depends on the structure of the tree (i.e., its optimality).
\end{lem}
\begin{proof}
    The proof of Lemma~\ref{lem:necessary:exact} from the supplementary online material appears in Appendix~\ref{appx:necessary:exact}.
\end{proof}

The state added by \ac{RRTstar} at each iteration, $\xnew$, is generated from a randomly sampled state, $\xrand$, and the nearest vertex in the existing tree,
\begin{equation}\label{eqn:back:nearest}
    \vnearest \coloneqq \argmin_{\statev \in \vertexSet}\set{ \norm{\xrand - \statev}{2} },
\end{equation}
through expansion and differential constraints (i.e., the $\mathtt{Steer}$ function).
Absent any constraints (i.e., in holonomic planning) this takes the form
\begin{equation}\label{eqn:back:steer}
    \xnew \coloneqq \argmin_{\statey \in \stateSet}\setst{\norm{\xrand - \statey}{2}}{\norm{\statey - \vnearest}{2} \leq \maxEdge},
\end{equation}
where $\maxEdge$ is a user-selected maximum edge length.

The number of tree vertices in the problem domain increases indefinitely with \ac{RRTstar} iterations.
With an infinite number of iterations, eventually all reachable states will be no more than $\maxEdge$ away from the nearest vertex in the tree.
After these $\iterThresh$ iterations, \emph{sampling} the omniscient set is a necessary condition to add a state from the omniscient set and improve the solution (Lemma~\ref{lem:necessary:exact:sample}).

\begin{lem}[The necessity of sampling states in the omniscient set in holonomic planning]\label{lem:necessary:exact:sample}
    Sampling the omniscient set, $\xrand\in\fTrueSet$, is a necessary condition for \ac{RRTstar} to improve the current solution to a holonomic problem, $\ccur$, after an initial $\iterThresh$ iterations,
    \begin{equation*}
        \forall \iter \geq \iterThresh,\, \cnext < \ccur \implies \xrand \in \fTrueSet,
    \end{equation*}%
    for any sample distribution that maintains a nonzero probability over the entire omniscient set.
    
    For simplicity, this statement is limited to holonomic planning but it can be extended to specific constraints with appropriate assumptions.
\end{lem}
\begin{proof}
    The proof of Lemma~\ref{lem:necessary:exact:sample} from the supplementary online material appears in Appendix~\ref{appx:necessary:sample}.
\end{proof}

This result provides an upper limit on the probability of \ac{RRTstar} improving a solution at any iteration (Theorem~\ref{thm:necessary:exact:prob}).

\begin{thm}[An upper bound on the probability of improving a solution to a holonomic planning problem given knowledge of the omniscient set]\label{thm:necessary:exact:prob}
    The probability that an iteration of \ac{RRTstar} improves the current solution to a holonomic problem, $\ccur$, is bounded by the probability of sampling the omniscient set, $\fTrueSet$,
    \begin{equation*}
        \forall \iter \geq \iterThresh,\, \prob{\cnext<\ccur} \leq \prob{\xrand\in\fTrueSet},
    \end{equation*}%
    for any iteration, $\iter$, after a sufficient vertex density is achieved in the initial $\iterThresh$ iterations.
    
    For simplicity, this statement is limited to holonomic planning but it can be extended to specific constraints by expanding Lemma~\ref{lem:necessary:exact:sample}.
\end{thm}

\begin{proof}
    Proof of Theorem~\ref{thm:necessary:exact:prob} follows directly from Lemma~\ref{lem:necessary:exact:sample}.
    Sampling a state in $\fTrueSet$ is a necessary but not sufficient condition to improve the solution after $\iterThresh$ iterations; therefore, the probability of sampling such a state bounds the probability of improving the solution.
\end{proof}

Knowledge of an omniscient set requires solving the planning problem; however, these results can be extended to estimates of the omniscient set defined by solution cost heuristics (Definition~\ref{defn:informed}).

\begin{defn}[Informed set]\label{defn:informed}
    Let $\fBelow{\statex}$ represent a heuristic estimate of the solution cost constrained to go through a state, $\statex\in\stateSet$. A heuristic estimate of the omniscient set can then be defined as
    \begin{equation*}
        \fBelowSet \coloneqq \setst{\statex \in \stateSet}{\fBelow{\statex} < \ccur}.
    \end{equation*}%
    Such a set will be referred to as an \emph{informed set}.
\end{defn}

There are an infinite number of potential informed sets for any planning problem and choosing the `best' set requires methods to quantify their performance.
In binary classification, estimates are evaluated in terms of their precision and recall (Fig.~\ref{fig:pr}).
Analogue terms can be defined in sampling-based planning to quantify the ability of informed sets to estimate the omniscient set (Definitions~\ref{defn:precision} and \ref{defn:recall}).

\begin{defn}[Precision]\label{defn:precision}
    The precision of an informed sampling technique is the probability that random samples drawn from the informed set could also be drawn from the omniscient set (e.g., the percentage of states drawn from the informed set, $\fBelowSet$, that belong to the omniscient set, $\fTrueSet$).
    For uniform sampling of an informed set, this is a ratio of measures,
    \begin{equation*}
        \mathrm{Precision}\left(\fBelowSet\right) \coloneqq \frac{\lebesgue{\fBelowSet\cap\fTrueSet}}{\lebesgue{\fBelowSet}}.
    \end{equation*}
    Any informed set with nonzero sampling probability that is a \emph{subset} of the omniscient set will have $100\%$ precision.
\end{defn}
\begin{defn}[Recall]\label{defn:recall}
    The recall of an informed sampling technique is the probability that random states drawn from the omniscient set could also be sampled from the informed set (e.g., the percentage of states that belong to the omniscient set, $\fTrueSet$, with a nonzero probability of being sampled from the informed set, $\fBelowSet$).
    For uniform sampling of an informed set, this is a ratio of measures,
    \begin{equation*}
        \mathrm{Recall}\left(\fBelowSet\right) \coloneqq \frac{\lebesgue{\fBelowSet\cap\fTrueSet}}{\lebesgue{\fTrueSet}}.
    \end{equation*}
    Any informed set with nonzero sampling probability that is a \emph{superset} of the omniscient set will have $100\%$ recall.
\end{defn}
\begin{figure}[tb]
    \centering
    \includegraphics[width=\columnwidth]{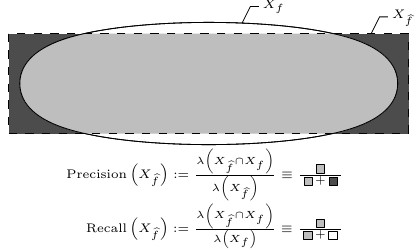}
    \caption{%
        An illustration of the \emph{precision} and \emph{recall} of estimating an oblong omniscient set, $\fTrueSet$, with a rectangular informed set, $\fBelowSet$.
        The informed set is coloured to highlight where it is correct (light grey) incorrect (dark grey) or missing the omniscient set (white).
        Precision is the likelihood of correctly sampling the omniscient set by sampling the informed set.
        Recall is the coverage of the omniscient set by the informed set.
        For uniform distributions, both these terms are ratios of Lebesgue measures.
    }
    \label{fig:pr}
\end{figure}%

Informed sets with $100\%$ recall (Definition~\ref{defn:admissible}) are important in almost-surely asymptotically optimal planning as less-than-perfect recall may exclude the optima to some problems.

\begin{defn}[Admissible informed set]\label{defn:admissible}
    A heuristic is said to be admissible if it never overestimates the true value of the function,
    \begin{equation*}
        \forall \statex \in \stateSet, \; \fBelow{\statex} \leq \fTrue{\statex}.
    \end{equation*}
    Any informed set defined by such an admissible heuristic will contain all possibly better solutions and have $100\%$ recall, i.e., $\fBelowSet \supseteq \fTrueSet.$
    
    This set will be referred to as an \emph{admissible} estimate of the omniscient set, or an \emph{admissible informed set}.
    If the heuristic is an admissible estimate of the cost function for all possible problems then the set will be referred to as a \emph{universally admissible informed set}.
\end{defn}

These definitions allow the probability of improving a solution to a holonomic problem to be bounded by the probability of sampling any admissible informed set (Lemmas~\ref{lem:necessary:heuristic} and \ref{lem:necessary:heuristic:sample} and Theorem~\ref{thm:necessary:heuristic:prob}).
The tightness of this bound will depend on the precision of the chosen estimate.
\begin{lem}[The necessity of adding states in an admissible informed set]\label{lem:necessary:heuristic}
    Adding a state from an admissible informed set, $\xnew\in\fBelowSet\supseteq\fTrueSet$, is a necessary condition for \ac{RRTstar} to improve the current solution, $\ccur$,
    \begin{equation*}
        \cnext < \ccur \implies \xnew \in \fBelowSet \supseteq \fTrueSet.
    \end{equation*}
\end{lem}
\begin{proof}
    Lemma~\ref{lem:necessary:heuristic} follows directly from Lemma~\ref{lem:necessary:exact} given that $\fBelowSet \supseteq \fTrueSet$.
\end{proof}

\begin{lem}[The necessity of sampling states in an admissible informed set in holonomic planning]\label{lem:necessary:heuristic:sample}
    Sampling an admissible informed set, $\xrand\in\fBelowSet\supseteq\fTrueSet$, is a necessary condition for \ac{RRTstar} to improve the current solution to a holonomic problem, $\ccur$, after an initial $\iterThresh$ iterations,
    \begin{equation*}
        \forall \iter \geq \iterThresh,\, \cnext < \ccur \implies \xrand \in \fBelowSet \supseteq \fTrueSet,
    \end{equation*}%
    for any sample distribution that maintains a nonzero probability over the entire informed set.
    
    For simplicity, this statement is limited to holonomic planning but it can be extended to specific constraints by expanding Lemma~\ref{lem:necessary:exact:sample}.
\end{lem}
\begin{proof}
    Lemma~\ref{lem:necessary:heuristic:sample} follows directly from Lemma~\ref{lem:necessary:exact:sample} given that $\fBelowSet \supseteq \fTrueSet$.
\end{proof}

\begin{thm}[An upper bound on the probability of improving a solution to a holonomic planning problem given knowledge of an admissible informed set]\label{thm:necessary:heuristic:prob}
    The probability that an iteration of \ac{RRTstar} improves the current solution to a holonomic problem, $\ccur$, is bounded by the probability of sampling an admissible informed set, $\fBelowSet\supseteq\fTrueSet$,
    \begin{equation*}
        \forall \iter \geq \iterThresh,\, \prob{\cnext<\ccur} \leq \prob{\xrand\in\fTrueSet}
                                                              \leq \prob{\xrand\in\fBelowSet},
    \end{equation*}%
    for any iteration, $\iter$, after an initial $\iterThresh$ iterations.
    
    For simplicity, this statement is limited to holonomic planning but it can be extended to specific constraints by expanding Lemma~\ref{lem:necessary:exact:sample}.
\end{thm}
\begin{proof}
    Theorem~\ref{thm:necessary:heuristic:prob} follows directly from Theorem~\ref{thm:necessary:exact:prob} given that $\fBelowSet \supseteq \fTrueSet$.
\end{proof}

\section{Prior Work Accelerating \acs{RRTstar} Convergence}\label{sec:lit}
A review of previous work to improve the convergence rate of \ac{RRTstar} is presented using the results and terminology of Section~\ref{sec:omni_inf}.
All these techniques attempt to increase the real-time rate of searching the omniscient set by exploiting additional information.
Most can be viewed as versions of sample biasing, sample rejection, and/or graph pruning (Sections~\ref{sec:lit:bias}--\ref{sec:lit:other}).

\subsection{Sample Biasing}\label{sec:lit:bias}
Increasing the likelihood of sampling an informed set improves \ac{RRTstar} performance.
This sample biasing creates a nonuniform sample distribution that will increase exploration of the informed set but invalidates the assumptions used to prove almost-sure asymptotic optimality.
One method to maintain these formal performance guarantees is to calculate the \ac{RGG} connection limit from a subset of samples that are uniformly distributed \cite{janson_ijrr15}.
This maintains almost-sure asymptotic optimality but increases the required number of rewirings.

It is common to bias sampling around the current solution.
This \emph{path biasing} increases the likelihood of sampling a state that can improve the current solution but reduces the likelihood of finding solutions in other homotopy classes (i.e., it increases precision by decreasing recall; Fig~\ref{fig:subsets}a).
The ratio of path biasing to global search is frequently a user-chosen parameter that must be tuned for each problem.

Akgun and Stilman \cite{akgun_iros11} use path biasing in their dual-tree version of \ac{RRTstar}.
Once an initial solution is found the algorithm spends a user-specified percentage of iterations refining the current solution.
It does this by explicitly sampling near a randomly selected state on the current path.
This increases the probability of improvement at the expense of decreasing the exploration of other homotopy classes.
Their algorithm also employs sample rejection in exploring the state space (see Section~\ref{sec:lit:rej}).
\begin{figure}[tb]
    \centering
    \includegraphics[page=1,scale=1]{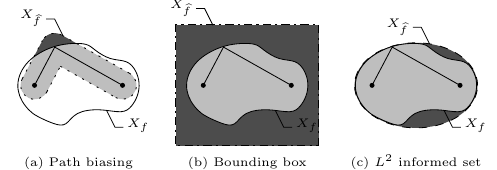} %
    \caption{%
    A illustration of the precision and recall of informed sampling techniques on the omniscient set depicted in Fig.~\ref{fig:intro:infsets}(b).
    The informed sets are coloured to highlight where they are correct (light grey), incorrect (dark grey), or missing the omniscient set (white).
    Path biasing, (a), generally has high precision but low recall, especially in the presence of multiple homotopy classes
    Global or bounded sampling, (b), generally has full recall but low precision, especially in large relative planning problems or high state dimensions.
    Direct sampling of the $L^2$ informed set, (c), has full recall and high precision, regardless of the size of the omniscient set and is exactly equal to the omniscient set in the absence of obstacles and constraints.
    }
    \label{fig:subsets}
\end{figure}%

Nasir et al. \cite{nasir_ijars13} combine path biasing with smoothing in their \acs{RRTstar}-Smart algorithm.
Solution paths are simplified and then used as biases for further sampling around the solution.
Their path smoothing rapidly improves the current solution but the path biasing decreases the likelihood of finding a solution in a different homotopy class.

Kiesel et al. \cite{kiesel_socs12} use a two-stage sampling process in their \emph{f-biasing} technique.
Samples are generated by randomly selecting a region of the planning problem and then uniformly sampling it.
The probability of selecting a region is calculated by solving a simple discretization of the planning problem with Dijkstra's algorithm \cite{dijkstra_59}.
The regions along the discrete solution are given a higher selection probability but all regions maintain a nonzero probability to compensate for the incompleteness of the discretization.
This technique provides a sampling bias for the entire \ac{RRTstar} search but once a solution is found it continues to sample states that cannot provide a better solution.
It is stated that almost-sure asymptotic optimality is maintained but it is not discussed how to modify the rewiring neighbourhood to do so.
\squeezeWidowedWords

Kim et al. \cite{kim_icra14} also use a two-stage sampling process in their Cloud \acs{RRTstar} algorithm.
They generate uniform samples from a series of collision-free, possibly overlapping, spheres defined by a \acl{GVG} \cite{choset_icra95}.
New spheres are added on solution paths and the probability of selecting them is updated so that samples from the homotopy class of the solution are biased around the path while maintaining the probability of sampling other homotopy classes.
Cloud \acs{RRTstar} successfully finds better solutions faster than other algorithms but continues to sample states that cannot improve the solution and its effect on almost-sure asymptotic optimality is not discussed.

\escapeParagraph
Unlike sample biasing methods, the direct informed sampling used by Informed \acs{RRTstar} does not consider states that are known to be unable to improve a solution.
It does result in a nonuniform sample distribution over the problem domain but it is still almost-surely asymptotically optimal as it has a uniform distribution in the informed set being searched.

\subsection{Sample Rejection}\label{sec:lit:rej}
Ignoring samples outside an informed set improves \ac{RRTstar} performance.
This sample rejection decreases the computational cost of states that cannot improve a solution but does not increase the probability of finding ones that can.
If this probability is low (i.e., if the informed set is small relative to the sampling domain) then convergence will not be improved (Fig.~\ref{fig:subsets}b).
It is shown that this probability decreases factorially with state dimension (i.e., faster than exponentially) in existing formulations of the holonomic minimum-path-length problem (Theorem~\ref{thm:curse}).

Akgun and Stilman \cite{akgun_iros11} use global rejection sampling in addition to sample biasing in their dual-tree algorithm.
As samples are drawn from the entire problem domain, performance will decrease rapidly as the solution improves and/or in large or high-dimensional planning problems.

Otte and Correll \cite{otte_tro13} use adaptive rejection sampling in their parallelized \ac{CFOREST} algorithm.
Samples are generated from a rectangular subset of the planning domain that bounds the ellipsoidal $L^2$ informed set and rejected using the $L^2$ heuristic.
This increases sampling precision and improves performance in large planning problems but its effectiveness still decreases factorially with state dimension (Theorem~\ref{thm:curse}).

\escapeParagraph
Unlike sample rejection methods, the direct informed sampling used by Informed \acs{RRTstar} maintains high precision and $100\%$ recall regardless of the relative sizes of the informed set and problem domain.
It focuses its search in response to solution improvements and does not decrease in effectiveness in large planning domains.
It scales more effectively than existing approaches to high-dimensional planning problems

\subsection{Graph Pruning}\label{sec:lit:prune}
Limiting the tree to an informed set improves \ac{RRTstar} performance.
This graph pruning removes states that can no longer improve the existing solution and reduces the computational cost of basic operations (e.g., nearest neighbour searches).
It can also be used reject potential new states given their connection and any constraints, e.g., \eqref{eqn:back:steer}.
After a sufficient number of iterations, this incremental pruning is equivalent in holonomic planning to rejection sampling with the same heuristic (Lemma~\ref{lem:necessary:heuristic:sample}) but with the additional computational costs of expanding towards the sample.

Karaman et al. \cite{karaman_icra11} use graph pruning to implement an online version of \ac{RRTstar} that improves solutions during path execution.
They remove vertices whose \emph{current} cost-to-come plus a heuristic estimate of cost-to-go is higher than the current solution.
As current cost-to-come overestimates a vertex's optimal cost-to-come (i.e., it is an inadmissible heuristic), this approach may erroneously remove vertices that could provide a better solution.
\squeezeWidowedWords

Arslan and Tsiotras \cite{arslan_icra13,arslan_icra15} combine incremental graph-pruning and incremental graph search techniques with \acp{RRG} \cite{karaman_ijrr11} to reject samples in their \RRTsharp{} algorithm.
This incremental pruning focuses the search but its performance will also decrease rapidly as the solution improves or when used on large or high-dimensional planning problems.
Some of the rejection criteria also use the current cost-to-come of vertices and may reject samples that could later improve the solution.

\escapeParagraph
Unlike rejecting states with incremental graph pruning, the direct informed sampling used by Informed \acs{RRTstar} wastes no computational effort on states that are known to be unable to improve the solution.
Its admissible graph pruning algorithm to remove unnecessary states also only removes vertices from the tree if doing so does not negatively affect the search.

\subsection{Other Techniques}\label{sec:lit:other}
Some techniques to improve \acs{RRT}/\acs{RRTstar} performance do not fit neatly into the previous categories.
Many of these methods could be further accelerated through direct informed sampling.

Urmson and Simmons \cite{urmson_iros03} uses rejection sampling to create a ``probabilistic implementation of heuristic search concepts'' in their \ac{hRRT}.
At each iteration, a uniformly distributed sample is probabilistically kept or rejected as a function of its heuristic value relative to the existing tree.
This iteratively biases \ac{RRT} expansion towards regions of the problem domain believed to contain high-quality solutions and often finds better solutions than \ac{RRT}, especially on problems with continuous cost functions (e.g., path length \cite{urmson_iros03}); however, it results in nonuniform sample distributions.

Ferguson and Stentz \cite{ferguson_iros06} recognize that an existing solution defines the set of states that could provide better solutions.
Their Anytime \acsp{RRT}s algorithm attempts to incrementally find better solutions by searching a decreasing series of these ellipses.
This shrinking search ignores some expensive solutions but does not guarantee better ones will be found.

Alterovitz et al. \cite{alterovitz_icra11} add path refinement to \ac{RRTstar} in their \ac{RRM} algorithm.
Once an initial solution is found, each iteration of \ac{RRM} either samples a new state or selects an existing state from the current solution and refines it.
Path refinement connects the selected state to its neighbours and results in a graph instead of a tree.
The ratio of refinement to exploration is a user-tuned parameter.

Shan et al. \cite{shan_iv14} find an initial solution with \ac{RRT}, simplify and rewire it using their \ac{RRTstar}\_S algorithm, and then continue the search with \ac{RRTstar}.
This can find better solutions faster than \ac{RRTstar} alone but the resulting search is not focused and continues to consider states that cannot provide better solutions.

Salzman and Halperin \cite{salzman_icra14} relax performance to asymptotic \emph{near} optimality in their \ac{LBTRRT}.
Rewirings are only considered if they are required to maintain the desired tolerance to the optimum.
This can reduce computational complexity but does not focus the search.

Devaurs et al. \cite{devaurs_afr15} use ideas from stochastic optimization to explore complex cost functions in their \ac{TRRTstar} and \ac{ATRRT} algorithms.
Transition tests accept or reject a potential new state depending on its cost relative to its parent.
These tests help reduce the \textit{integral} or \textit{mechanical work} of the path in a cost space; however, for problems seeking to minimize path length are equivalent to graph pruning.

\escapeParagraph
These algorithms, and those designed for more advanced purposes (e.g., \RRTx{} \cite{otte_ijrr16}), can be improved with the direct informed sampling and admissible graph pruning techniques illustrated in Informed \acs{RRTstar}.

\subsection{Direct Informed Sampling for Path Length}
This paper presents Informed \ac{RRTstar} as a demonstration of how direct sampling of $L^2$ informed sets increases the rate at which \ac{RRTstar} improves solutions for problems seeking to minimize path length.
Unlike sample biasing, this approach considers all homotopy classes that could provide better solutions (i.e., $100\%$ recall) while maintaining uniform sample distribution over a subplanning problem.
Unlike sample rejection or graph pruning, it is effective regardless of the relative size of the informed set or the state dimension (i.e., high precision).
In situations where the heuristic does not provide substantial information (i.e., small planning problems and/or large informed sets), it performs identically to \ac{RRTstar}.

\section{The \texorpdfstring{$L^2$}{L2} Informed Set}\label{sec:l2}
A universally admissible heuristic is well defined for problems seeking to minimize path length in $\real^\dimension$ and is commonly used in sampling-based planners (e.g., \cite{ferguson_iros06,akgun_iros11,otte_tro13}).
The cost of a solution constrained to pass through any state, $\statex\in\stateSet$, is bounded from below by the $L^2$ norm (i.e., Euclidean distance) between it, the start, $\xstart$, and the goal, $\xgoal$,
\begin{equation}\label{eqn:fBelow}
    \fBelow{\statex} = \norm{\statex - \xstart}{2} + \norm{\xgoal - \statex}{2}.
\end{equation}%
The set of states that could provide a better solution than the current solution cost, $\ccur$, can then be referred to as the $L^2$ informed set,
\begin{equation*}
    \fBelowSet = \setst{\statex \in \freeSet}{\norm{\statex - \xstart}{2} + \norm{\xgoal - \statex}{2} < \ccur}.
\end{equation*}%
This informed set is a universally admissible estimate of the omniscient set and is exact in the absence of obstacles and constraints (i.e., it is sharp over all minimum-path-length problems).
The size of this informed set will decrease as solutions improve.

\begin{figure}[tb]
    \centering
    \includegraphics[scale=1]{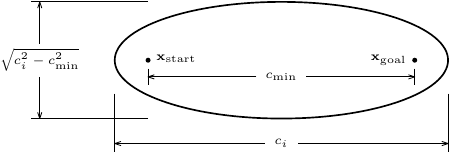} %
    \caption{%
    The $L^2$ informed set, $\fBelowSet$, for a $\real^2$ problem seeking to minimize path length is an ellipse with the initial state, $\xstart$, and the goal state, $\xgoal$, as focal points.
    The shape of the ellipse depends on both the initial and goal states, the theoretical minimum cost between the two, $\cmin$, and the cost of the best solution found to date, $\ccur$.
    The eccentricity of the ellipse is given by $\cmin/\ccur$.
    \squeezeWidowedWords
    }
    \label{fig:ellipse}
\end{figure}%
The $L^2$ informed set is the intersection of the free space, $\freeSet$, and a $\dimension$-dimensional hyperellipsoid symmetric about its transverse axis (i.e., a prolate hyperspheroid),
\begin{equation*}
    \fBelowSet = \freeSet \cap \phsSet,
\end{equation*}
where
\begin{equation*}
    \phsSet \coloneqq \setst{\statex \in \real^\dimension}{\norm{\statex - \xstart}{2} + \norm{\xgoal - \statex}{2} < \ccur}.
\end{equation*}

The prolate hyperspheroid has focal points at $\xstart$ and $\xgoal$, a transverse diameter of $\ccur$, and conjugate diameters of $\sqrt{\ccur^2 - \cmin^2}$, where
\begin{equation*}
    \cmin \coloneqq \norm{\xgoal-\xstart}{2},
\end{equation*}
is the theoretical minimum cost (Fig.~\ref{fig:ellipse}).
The Lebesgue measure of the informed set is
\begin{equation}\label{eqn:phsMeasure}
    \lebesgue{\fBelowSet} \leq \lebesgue{\phsSet} = \frac{\ccur \left( \ccur^2 - \cmin^2 \right)^{\frac{\dimension-1}{2}} \unitBall{\dimension}}{2^\dimension},
\end{equation}
where $\unitBall{\dimension}$ is the Lebesgue measure of a $\dimension$-dimensional unit ball,
\begin{equation}\label{eqn:ballMeasure}
    \unitBall{\dimension} \coloneqq \frac{\pi^{\frac{\dimension}{2}}}{\gammaFunc{\frac{\dimension}{2} + 1}},
\end{equation}
and $\gammaFunc{\cdot}$ is the gamma function, an extension of factorials to real numbers \cite{gamma_function}.

The probability of uniformly sampling this informed set by sampling any superset (e.g., a bounding box), $\samplingSet\supseteq\fBelowSet$, can be written as a ratio of measures,
\begin{align}\label{eqn:sampleProb}
    &\probst{\xrand \in \fBelowSet}{\xrand\sim\uniform{\samplingSet}}
    \leq \frac{\lebesgue{\phsSet}}{\lebesgue{\samplingSet}} \nonumber\\
    &\qquad\qquad\qquad\qquad\qquad\quad
    {}= \frac{\pi^{\frac{\dimension}{2}} \ccur \left( \ccur^2 - \cmin^2 \right)^{\frac{\dimension-1}{2}}}{2^\dimension \gammaFunc{\frac{\dimension}{2} + 1} \lebesgue{\samplingSet}},
\end{align}
which can be combined with Theorem~\ref{thm:necessary:heuristic:prob} to bound the probability of improving a solution to a holonomic problem,
\begin{align}\label{eqn:betterProb}
    \forall \iter \geq \iterThresh,\, & \probst{\cnext < \ccur}{\xrand\sim\uniform{\samplingSet}} \nonumber\\
    &\qquad\qquad\qquad\quad
    \leq \frac{\pi^{\frac{\dimension}{2}} \ccur \left( \ccur^2 - \cmin^2 \right)^{\frac{\dimension-1}{2}}}{2^\dimension \gammaFunc{\frac{\dimension}{2} + 1} \lebesgue{\samplingSet}}.
\end{align}

\begin{figure*}[tb]
    \centering
    \includegraphics[width=\textwidth]{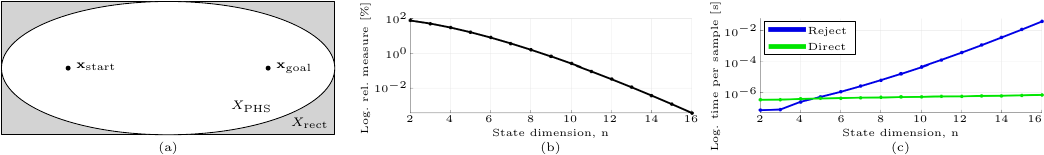}
    \caption{%
        An illustration of state dimension on problems seeking to minimize path length.
        The best case performance of an admissible rectangular sampling, e.g., \cite{otte_tro13}, occurs when the rectangle tightly bounds the prolate hyperspheroid defined by the current solution cost, $\rectSet\supset\phsSet\supseteq\fBelowSet$, (a).
        The probability of sampling this $L^2$ informed set (i.e., its relative measure) decreases \emph{factorially} (i.e., faster than exponentially) with state dimension, $\dimension$, (b), meaning that existing formulations of \ac{RRTstar} do not scale effectively to high state dimensions.
        Direct informed sampling, \algo{algo:infset}, scales more efficiently as illustrated by the average per-sample time versus state dimension, (c).
        Samples in the unit $\dimension$-ball for \algo{algo:phs} were generated with Boost 1.58.
    }
    \label{fig:sampleTheory}
\end{figure*}%
This probability becomes arbitrarily small for
\begin{inparaenum}[(i)]
    \item costs, $\ccur$, near the theoretical limit, $\cmin$,\label{item:reject:cost}
    \item large sampling domains, $\lebesgue{\samplingSet}$, or\label{item:reject:size}
    \item high state dimensions, $\dimension$.\label{item:reject:state}
\end{inparaenum}
While the solution cost and sampling domain size may vary during the search of a problem, the state dimension is constant throughout.
This motivates investigating the effect of state dimension on existing formulations of the holonomic minimum-path-length planning problem (Theorem~\ref{thm:curse}).

\begin{thm}[The minimum-path-length curse of dimensionality]\label{thm:curse}
    The probability that \ac{RRTstar} improves a solution to holonomic problems seeking to minimize path length decreases \emph{factorially} (i.e., faster than exponentially) as state dimension increases,
    \begin{align}\label{eqn:thm:curse}
        \forall \iter \geq \iterThresh,\, \probst{\cnext<\ccur}{\xrand\sim\uniform{\rectSet}} \leq& \frac{\pi^{\frac{\dimension}{2}}}{2^\dimension\gammaFunc{\frac{\dimension}{2} + 1}},
    \end{align}
    when uniformly sampling a (hyper)rectangle bounding the $L^2$ informed set, $\rectSet\supset\phsSet\supseteq\fBelowSet\supseteq\fTrueSet$.
    
    For simplicity, this statement is limited to holonomic planning but it can be extended to specific constraints by expanding Lemma~\ref{lem:necessary:exact:sample}.
\end{thm}
\begin{proof}
    Theorem~\ref{thm:curse} is proven for \ac{RRTstar} but holds for any algorithm for which an equivalent to Theorem~\ref{thm:necessary:heuristic:prob} exists.

    The smallest possible $\rectSet$ that completely contains $\phsSet$ is a (hyper)rectangle with widths corresponding to the diameters of the prolate hyperspheroid (Fig.~\ref{fig:sampleTheory}a).
    The measure of any $\rectSet\supset\phsSet$ is therefore bounded from below as
    \begin{equation}\label{eqn:thm:curse:tightMeasure}
        \lebesgue{\rectSet} \geq \ccur \left( \ccur^2 - \cmin^2 \right)^{\frac{\dimension-1}{2}}.
    \end{equation}
    When substituted into \eqref{eqn:betterProb} this gives
    \begin{equation*}
        \forall \iter \geq \iterThresh,\, \probst{\cnext < \ccur}{\xrand\sim\uniform{\rectSet}} \leq \frac{\pi^{\frac{\dimension}{2}}}{2^\dimension\gammaFunc{\frac{\dimension}{2} + 1}},
    \end{equation*}
    proving Theorem~\ref{thm:curse} for all rectangular sets, $\rectSet$, such that $\rectSet \supset \phsSet \supseteq \fBelowSet\supseteq\fTrueSet$.
\end{proof}

Theorem~\ref{thm:curse} is an upper bound on the utility of rectangular rejection sampling in holonomic planning and is illustrated by plotting \eqref{eqn:thm:curse} versus state dimension (Fig.~\ref{fig:sampleTheory}b).
The results show that while rectangular rejection sampling may be $79\%$ successful in $\real^2$, its success decreases factorially as state dimension increases and is only $2\%$ in $\real^8$ and $4 \times 10^{-4}\%$ in $\real^{16}$.
These numbers represent the \emph{best-case} for rectangular rejection sampling and actual performance will depend on the size and orientation of the informed set relative to the sampling domain.
This motivates a need for a direct method to sample the prolate hyperspheroid regardless of size, orientation, and state dimension.

\subsection{Direct Sampling}\label{sec:l2:sample}
A direct method to generate uniformly distributed samples in the $L^2$ informed set is adapted from techniques to sample hyperellipsoids \cite{sun_fusion02}.

Let $\ellipseMatrix\in\real^{\dimension\times \dimension}$ be a symmetric, positive-definite matrix (the hyperellipsoid matrix) such that the interior of a hyperellipsoid, $\ellipseSet$, is defined as
\begin{equation}\label{eqn:SDefn}
    \ellipseSet \coloneqq \set{\statex \in \real^\dimension\,\middle|\,\left( \statex - \xcentre \right)^T\ellipseMatrix^{-1}\left( \statex - \xcentre \right) < 1},
\end{equation}%
where $\xcentre$ is the centre point of the hyperellipsoid.
Uniformly distributed samples in the hyperellipsoid, $\xellipse\sim\uniform{\ellipseSet}$, can be generated from uniformly distributed samples in the interior of a unit $\dimension$-dimensional ball, $\xball\sim\uniform{\ballSet}$, by
\begin{equation}\label{eqn:transformDefn}
    \xellipse = \mathbf{L} \xball + \xcentre,
\end{equation}%
where $\mathbf{L}\in\real^{\dimension\times \dimension}$ is the lower-triangular Cholesky decomposition of the hyperellipsoid matrix such that
\begin{equation*}
    \mathbf{L}\mathbf{L}^T \equiv \ellipseMatrix
\end{equation*}%
and
\begin{equation*}
    \ballSet \coloneqq \setst{\statex\in\real^\dimension}{\norm{\statex}{2} < 1}.
\end{equation*}

For hyperellipsoids with orthogonal axes, there exists a coordinate frame in which the hyperellipsoid matrix is diagonal,
\begin{equation*}
    \ellipseMatrix' \coloneqq \diag\left( \radius_1^2, \radius_2^2, \ldots, \radius_n^2 \right),
\end{equation*}
where $\radius_\counterj$ is the radius of $\counterj$-th axis of the hyperellipsoid and $\diag\left(\cdot\right)$ constructs a diagonal matrix.
A rotation from this hyperellipsoid-aligned frame to the world frame, $\ellipseRotation\in SO\left(\dimension\right)$, can be used to write \eqref{eqn:SDefn} in terms of $\ellipseMatrix'$ as
\begin{align}\label{eqn:transformFinal}
    \ellipseSet \coloneqq &\left\lbrace\statex \in \real^\dimension\;\;\middle|
        \vphantom{\left( \statex - \xcentre \right)^T\ellipseRotation\mathbf{S'}^{-1}\ellipseRotation^T\left( \statex - \xcentre \right) < 1}\right.
        \nonumber\\
        &\;\left.\vphantom{\statex \in \real^\dimension}
    \left( \statex - \xcentre \right)^T\ellipseRotation\mathbf{S'}^{-1}\ellipseRotation^T\left( \statex - \xcentre \right) < 1\right\rbrace,\nonumber
\shortintertext{and \eqref{eqn:transformDefn} as}
    &\qquad\xellipse = \ellipseRotation\mathbf{L}' \xball + \xcentre,
\end{align}
given the orthogonality of rotation matrices, $\ellipseRotation^{-1}\equiv\ellipseRotation^T$, and that $\mathbf{L}'\mathbf{L}'^T \equiv \ellipseMatrix'$.

The rotation between frames can be solved directly as a general Wahba problem \cite{wahba_siam65} even when underspecified \cite{ruiter_jas14}.
Generally, the rotation matrix from one set of axes, $\set{\axis_{j}}$, to another set of axes, $\set{\mathbf{b}_{j}}$, is given by 
\begin{equation} \label{eqn:svd}%
    \mathbf{C}_{\rm ba} = \mathbf{U}\boldsymbol{\Lambda}\mathbf{V}^{T},
\end{equation}%
where $\boldsymbol{\Lambda}\in \real^{\dimension\times \dimension}$ is
\begin{equation*}
    \boldsymbol{\Lambda} \coloneqq \diag\left( 1, \ldots, 1, \det\left(\mathbf{U}\right) \det\left(\mathbf{V}\right) \right),
\end{equation*}
and $\det\left(\cdot\right)$ is the matrix determinant.
The terms $\mathbf{U} \in \real^{\dimension\times \dimension}$ and $\mathbf{V} \in \real^{\dimension\times \dimension}$ are unitary matrices such that $\mathbf{U}\boldsymbol{\Sigma}\mathbf{V}^T \equiv \mathbf{M}$ via singular value decomposition and $\mathbf{M}\in\real^{\dimension\times\dimension}$ is given by the outer product of the $\counterj\leq \dimension$ corresponding axes,
\begin{equation}\label{eqn:MDefn}
    \mathbf{M} \coloneqq \left[\axis_{1}, \axis_{2}, \ldots, \axis_{j} \right] \left[\mathbf{b}_{1}, \mathbf{b}_{2}, \ldots \mathbf{b}_{j}\right]^T.
\end{equation}

In problems seeking to minimize path length, the hyperellipsoid is a prolate hyperspheroid described by
\begin{align}
    \xcentre &\coloneqq \frac{\xstart + \xgoal}{2},\label{eqn:phs:xcentre}\\
    \mathbf{S'} &\coloneqq \diag \left( \frac{\ccur^2}{4}, \frac{\ccur^2 - \cmin^2}{4}, \ldots, \frac{\ccur^2 - \cmin^2}{4} \right),\nonumber\\
\shortintertext{and therefore,}
    \mathbf{L}' &= \diag \left( \frac{\ccur}{2}, \frac{\sqrt{\ccur^2 - \cmin^2}}{2}, \ldots, \frac{\sqrt{\ccur^2 - \cmin^2}}{2} \right)\label{eqn:phs:L}.
\end{align}%
Its local coordinate system is underspecified in the conjugate directions due to symmetry, making \eqref{eqn:MDefn} just
\begin{equation}\label{eqn:phs:M}%
    \mathbf{M} = \axis_1\mathbf{1}_1^T,
\end{equation}%
where $\mathbf{1}_1$ the first column of the identity matrix and the transverse axis in the world frame is
\begin{equation*}%
    \axis_{1} = \left( \xgoal - \xstart \right)/\norm{\xgoal - \xstart}{2}.
\end{equation*}%

Samples distributed uniformly in the $L^2$ informed set, $\fBelowSet=\phsSet\cap\freeSet$, can therefore be generated by using \eqref{eqn:transformFinal} to transform samples drawn uniformly from a unit $\dimension$-ball.
These samples are mapped to the prolate hyperspheroid through scaling, \eqref{eqn:phs:L}, rotation, \eqref{eqn:svd} and \eqref{eqn:phs:M}, and translation, \eqref{eqn:phs:xcentre}.

Sun and Farooq \cite{sun_fusion02} investigate various methods to generate samples in hyperellipsoids and provide the following lemma regarding the uniform sample density of this technique.
\begin{lem}[The uniform distribution of samples transformed into a hyperellipsoid from a unit $\dimension$-ball. Originally Lemma 1 in \cite{sun_fusion02}]\label{lem:uniform}
    If the random points distributed in a hyperellipsoid are generated from the random points uniformly distributed in a hypersphere through a linear invertible nonorthogonal transformation, then the random points distributed in the hyperellipsoid are also uniformly distributed.
\end{lem}
\begin{proof}
    For brevity, \cite{sun_fusion02} only presents anecdotal proofs of Lemma~\ref{lem:uniform}.
    The full proof from the supplementary online material appears in Appendix~\ref{appx:uniform}.
\end{proof}

\begin{algorithm}[tp]%
    \algorithmStyle
    \caption{$\mathtt{Sample}\left(\xstart\in\stateSet,\; \xgoal\in\stateSet,\; \ccur\in\positiveReal \right)$}\label{algo:infset}
    \Repeat{$\xrand \in \freeSet\cap\phsSet$}
    {
        \If{$\lebesgue{\phsSet} < \lebesgue{\stateSet}$}
        {
            $\xrand \gets \mathtt{SamplePHS}\left(\xstart,\xgoal,\ccur\right)$\;
        }
        \Else
        {
            $\xrand \gets \mathtt{SampleProblem}\left(\stateSet\right)$\;
        }
    }
    \Return{$\xrand$}\;
\end{algorithm}%
\begin{algorithm}[tp]%
    \algorithmStyle
    \caption{$\mathtt{SamplePHS}\left(\xstart\in\stateSet,\; \xgoal\in\stateSet,\; \ccur\in\positiveReal \right)$}\label{algo:phs}
    $\cmin \gets \norm{\xgoal - \xstart}{2}$\; \label{algo:phs:startStatic}
    $\xcentre \gets \left(\xstart + \xgoal\right)/2$\;
    $\axis_1 \gets \left(\xgoal - \xstart\right)/\cmin$\;
    $\set{\mathbf{U},\,\mathbf{V}} \gets \mathtt{SVD}\left(\axis_1\mathbf{1}_1^T\right)$\;
    $\boldsymbol{\Lambda} \gets \diag\left( 1, \ldots, 1, \det\left(\mathbf{U}\right) \det\left(\mathbf{V}\right) \right)$\;
    $\ellipseRotation \gets \mathbf{U}\boldsymbol{\Lambda}\mathbf{V}^{T}$\; \label{algo:phs:endStatic}
    $\radius_{1} \gets \ccur/2$\;
    $\left\lbrace \radius_\counterj\right\rbrace_{\counterj = 2,\ldots,\dimension} \gets \left(\sqrt{\ccur^2 - \cmin^2}\right)/2$\;
    $\mathbf{L} \gets \diag\left(\radius_1, \radius_2, \ldots, \radius_\dimension\right)$\;
    $\xball \gets \mathtt{SampleUnitBall}\left(\dimension\right)$\;
    $\xrand \gets \ellipseRotation\mathbf{L}\xball + \xcentre$\;
    \Return{$\xrand$}\;
\end{algorithm}%
\subsubsection{Algorithm}\label{sec:l2:sample:algo}
The $L^2$ informed set is an arbitrary intersection of the prolate hyperspheroid and the problem domain.
It can be sampled efficiently by considering the relative measure of the two sets and sampling the smaller set until a sample belonging to both sets is found.
These procedures are presented algorithmically in \algoAnd{algo:infset}{algo:phs} and are publicly available in \ac{OMPL}.
Note that for most problems \algolines{algo:phs}{startStatic}{endStatic} are constant and only need to be calculated once.

The function $\mathtt{SVD}\left(\cdot\right)$ denotes the singular value decomposition of a matrix and $\mathtt{SampleUnitBall}\left(\dimension\right)$ returns uniformly distributed samples from the interior of an $\dimension$-dimensional unit ball.
The measure of the prolate hyperspheroid, $\lebesgue{\phsSet}$, is given by \eqref{eqn:phsMeasure} and \texttt{SampleProblem} returns samples uniformly distributed over the entire planning domain.
Implementations of \texttt{SVD} and \texttt{SampleUnitBall} can be found in common C++ libraries.

\subsubsection{Practical Performance}\label{sec:l2:sample:exp}
Direct informed sampling (\algo{algo:infset}) is compared to the \emph{best-case} performance of rectangular rejection sampling.
The average computational time required to find a sample in the $L^2$ informed set is calculated by generating $10^6$ samples at each dimension (Fig.~\ref{fig:sampleTheory}c).
The results show that while rejection sampling may outperform direct informed sampling in low state dimensions (e.g., $\real^2$: $7.3\times10^{-8}$ vs. $3.5\times10^{-7}$ seconds), it becomes orders of magnitude slower as state dimension increases (e.g., $\real^{16}$: $4.0\times10^{-2}$ vs. $7.2\times10^{-7}$ seconds).
These per-sample times are small but significant. %
Generating $10^5$ samples in $\real^{16}$ requires less than a second with direct informed sampling ($7.2\times10^{-2}$ seconds) but over an hour with rectangular rejection sampling ($3953$ seconds).

This experiment represents optimistic results for both constant (e.g., the problem domain) and adaptive (e.g., \cite{otte_tro13}) rectangular rejection sampling.
Constant sampling domains rarely provide tight bounds on the informed set and will generally have higher rejection rates than the experiment.
Adaptive sampling domains may tightly bound the informed set but must account for its alignment relative to the state space.
This requires either a larger rectangular sampling domain or a rotation between frames that increases the rejection rate or computational cost compared to the experiment, respectively.

\subsection{Extension to Multiple Goals}\label{sec:l2:multi}
Many planning problems seek the minimum-length path that connects a start to any state in a goal region, $\goalSet$.
In these situations the omniscient set is all states that could provide a better solution to \emph{any} goal.
The multigoal $L^2$ informed set is
\begin{align*}
    \fBelowSet \coloneqq &\left\lbrace\statex \in \freeSet\;\;\middle|\;\;\norm{\statex - \xstart}{2} + \norm{\xgoalj - \statex}{2} < \ccur
        \vphantom{\mbox{for any}\;\;\xgoalj\in\goalSet}\right.\nonumber\\
        &\qquad\qquad\qquad\qquad\qquad\quad\left.\vphantom{\statex \in \freeSet \norm{\statex - \xstart}{2} + \norm{\xgoalj - \statex}{2}  < \ccur}
    \mbox{for any}\;\;\xgoalj\in\goalSet\right\rbrace.
\end{align*}%

For a countable goal region, $\goalSet \coloneqq \set{\xgoalj}_{j=1}^\numGoals$, this set is the union of the individual informed sets of each goal,
\begin{equation*}
    \fBelowSet = \bigcup_{\counterj=1}^\numGoals\fBelowSetj,
\end{equation*}
where $\numGoals$ is the number of goals and
\begin{equation*}
    \fBelowSetj \coloneqq \setst{\statex \in \freeSet\hspace{-1pt}}{\hspace{-1pt}\norm{\statex - \xstart}{2} + \norm{\xgoalj - \statex}{2} < \ccur},
\end{equation*}
is the $L^2$ informed set of an individual $\pair{\xstart}{\xgoalj}$ pair.
If the individual informed sets do not intersect, then a uniform sample distribution can be generated by randomly selecting an individual subset, $\counterj$, in proportion to its relative measure,
\begin{equation*}
    \pdfSymb\left(1\leq \counterj\leq \numGoals\right)\coloneqq\frac{\lebesgue{\fBelowSetj}}{\sum_{\counterk=1}^{\numGoals}\lebesgue{\fBelowSetk}},
\end{equation*}
and then generating a uniformly distributed sample inside the selected subset, $\fBelowSetj$.

\begin{algorithm}[tp]%
    \algorithmStyle
    \caption{$\mathtt{Sample}\hspace{-1pt}\left(\xstart\in\stateSet,\; \goalSet\subset\stateSet,\; \ccur\in\positiveReal\right)$}\label{algo:multigoal}
    \Repeat{$\xrand\in\freeSet \cap \newAlgoLine{\left(\bigcup_{\counterj=1}^{\numGoals}\phsSetj\right)}$\\
            \skipln$\qquad\qquad\qquad\qquad\qquad$\newAlgoLine{{\rm\bf and} $\mathtt{KeepSample}\left(\xrand, \xstart, \goalSet, \ccur\right)$}}
    {
        \If{$\newAlgoLine{\frac{1}{\numGoals}\sum_{\counterj=1}^{\numGoals}\lebesgue{\phsSetj}} < \lebesgue{\stateSet}$}
        {
            \newAlgoLine{$\xgoalj \gets \mathtt{RandomGoal}\left(\xstart, \goalSet, \ccur\right)$\;}
            $\xrand \gets \mathtt{SamplePHS}\left(\xstart,\newAlgoLine{\xgoalj},\ccur\right)$\;
        }
        \Else
        {
            $\xrand \gets \mathtt{SampleProblem}\left(\stateSet\right)$\;
        }
    }
    \Return{$\xrand$}\;
\end{algorithm}%
\begin{algorithm}[tp]%
    \algorithmStyle
    \caption{$\mathtt{RandomGoal}\left(\xstart\in\stateSet,\; \goalSet\subset\stateSet,\; \ccur\in\positiveReal\right)$}\label{algo:randomGoal}
    $a \gets 0$\;
    \ForAll{$\xgoalk \in \goalSet$}
    {
        $a \gets a + \lebesgue{\phsSetk}$\;
    }
    $p \gets \uniformSymb\left[0,1\right]$\;
    $j \gets 0$\;
    \Repeat{$p \leq 0$}
    {
        $j \gets j + 1$\;
        $p \gets p - \lebesgue{\phsSetj}/a$\;
    }
    \Return{$\xgoalj$}\;
\end{algorithm}%

If individual sets do intersect, then this approach will oversample states that belong to multiple sets (Fig.~\ref{fig:multigoal}a).
In these situations, uniform sample density can be maintained by probabilistically rejecting samples in proportion to their membership in individual sets.
This creates a uniform sample distribution for multigoal $L^2$ informed sets defined by arbitrarily overlapping individual informed sets (Fig.~\ref{fig:multigoal}b).

\subsubsection{Algorithm}\label{sec:l2:mult:algo}
The algorithm is described in \algos{algo:multigoal}{algo:randomKeep} as modifications to the sampling technique for a single-goal $L^2$ informed set, with changes highlighted in red (cf. \algo{algo:infset}).
The measure of individual informed sets, $\lebesgue{\phsSetj}$, is calculated from \eqref{eqn:phsMeasure} using the appropriate goal, $\xgoalj$.
This same technique can also be applied to problems with a countable start region.
\begin{algorithm}[tp]%
    \algorithmStyle
    \caption[]{%
                $\mathtt{KeepSample}\left(\xrand\in\stateSet,\; \xstart\in\stateSet,\right.$\captionbreak
                $\hphantom{\mathbf{Alg. \ref{algo:randomKeep}:} \mathtt{KeepSample}(\xrand\in\stateSet,\quad}
                \left.\goalSet\subset\stateSet,\; \ccur\in\positiveReal\right)$
              }\label{algo:randomKeep}
    $a \gets 0$\;
    \ForAll{$\xgoalk \in \goalSet$}
    {
        \If{$\norm{\xrand - \xstart}{2} + \norm{\xgoalk - \xrand}{2} < \ccur$}
        {
            $a \gets a + 1$\;
        }
    }
    $p \gets \uniformSymb\left[0,1\right]$\;
    \Return{$p \leq 1/a$}\;
\end{algorithm}%
\begin{figure}[tp]
    \centering
    \includegraphics[page=1,width=\columnwidth]{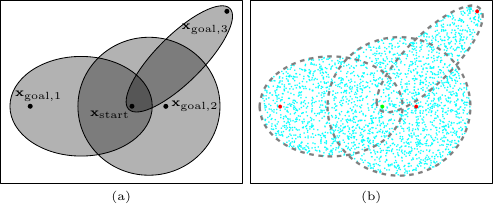}
    \caption{%
            An illustration of the multigoal $L^2$ informed set for a problem seeking to minimize path length from a start at $\left[0, 0\right]^T$, to any of three goals at $\left[ -0.75,\, 0 \right]^T$, $\left[ 0.25,\, 0 \right]^T$, and $\left[ 0.7,\, 0.7 \right]^T$, and a current solution cost of $\ccur = 1.05$.
            Each ellipse illustrates the $L^2$ informed set for a start-goal pair.
            Combining the uniform distributions of these individuals (light grey) would result in a \emph{nonuniform} distribution (dark grey), (a).
            By probabilistically rejecting samples in proportion to their individual membership, \algo{algo:multigoal} uniformly samples complex sets of arbitrary intersections, as illustrated with $2500$ random samples, (b).
            \squeezeWidowedWords
    }
    \label{fig:multigoal}
\end{figure}%
\section{Informed \acs{RRTstar}}\label{sec:inf}
Informed \ac{RRTstar} is an extension of \ac{RRTstar} that demonstrates how informed sets can be used to improve anytime almost-surely asymptotically optimal planning.
It performs the same as \ac{RRTstar} until a solution is found after which the search is focused to the informed set through direct informed sampling and admissible graph pruning (Fig.~\ref{fig:example}).
This increases the likelihood of sampling states that can improve the solution and increases the convergence rate towards the optimum regardless of the relative size of the informed set (e.g., near-minimum solutions or large problem domains) or the state dimension.

Informed \acs{RRTstar} uses direct informed sampling (\algo{algo:multigoal}), admissible graph pruning (Section~\ref{sec:inf:prune}), and an updated calculation of the rewiring neighbourhood (Section~\ref{sec:inf:rewire}) to focus the search.
The complete algorithm is presented in \algoAnd{algo:inf_rrtstar}{algo:prune} as modifications to \ac{RRTstar}, with changes highlighted in red.
It can also be integrated into other sampling-based planners, such as \ac{RRT}\textsuperscript{X} \cite{otte_ijrr16} and \ac{BITstar} \cite{gammell_icra15,gammell_phd17,gammell_ijrr18}.

At each iteration, Informed \ac{RRTstar} calculates the current best solution (\algoline{algo:inf_rrtstar}{min}) from the vertices in the goal region (\algo{algo:inf_rrtstar},~Lines~\ref{algo:inf_rrtstar:init}, \ref{algo:inf_rrtstar:goalStart}--\ref{algo:inf_rrtstar:goalEnd}).
This defines a shrinking $L^2$ informed set that is used to both focus sampling (\algoline{algo:inf_rrtstar}{sample}; \algo{algo:multigoal}) and prune the graph (\algoline{algo:inf_rrtstar}{prune}; \algo{algo:prune}).
This process continues for as long as time allows or until a suitable solution is found.

Informed \acs{RRTstar} retains the probabilistic completeness and almost-sure asymptotically optimality of \ac{RRTstar}.
It is probabilistically complete since it does not modify the search for an initial solution.
It is almost-surely asymptotically optimal as it maintains a uniform sample distribution over a subset of the planning problem in which it uses a local rewiring neighbourhood that satisfies the bounds presented in \cite{karaman_ijrr11}.

\begin{figure*}[tbp]
    \centering
    \includegraphics[page=1,width=\textwidth]{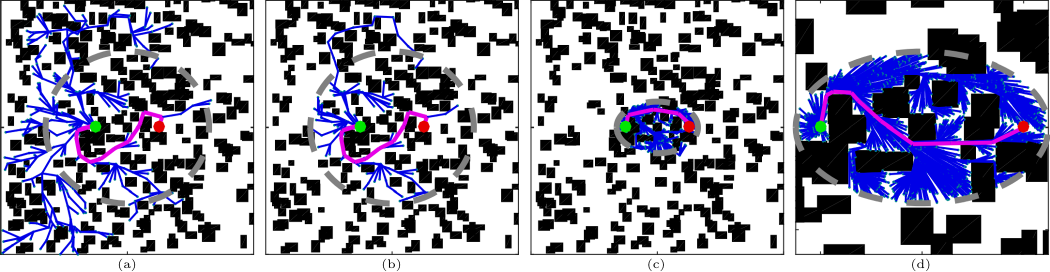}
    \caption{%
            An example of how Informed \acs{RRTstar} uses the current solution to focus search to the $L^2$ informed set.
            After an unfocused search for an initial solution, (a), Informed \acs{RRTstar} prunes the graph of unnecessary states and redefines the search domain to the current $L^2$ informed set, (b).
            Samples are then generated directly from this informed set, avoiding those that are known to be unable to improve the current solution.
            This reduced search space increases the likelihood of finding an improved solution, which in turn further reduces the search space, (c).
            This results in an algorithm that focuses the search to the subproblem given by the current solution, shown enlarged in (d), even as the subproblem decreases with further improvement.
            }
    \label{fig:example}
\end{figure*}
\begin{algorithm}[tb]%
    \algorithmStyle
    \caption{Informed \acs{RRTstar}$\left(\xstart\in\freeSet,\; \goalSet\subset\stateSet \right)$}\label{algo:inf_rrtstar}
    $\vertexSet \gets \left\lbrace \xstart \right\rbrace;\;$
    $\edgeSet \gets \emptyset;\;$
    $\treeGraph = \left( \vertexSet, \edgeSet \right)$\;
    \newAlgoLine{$\solnVertices \gets \emptyset$\;} \label{algo:inf_rrtstar:init}
    \For{$\iter = 1 \ldots \totalSamples$}
    {
        \newAlgoLine{$\ccur \gets \min_{\vgoal\in\solnVertices} \set {\gAbove{\vgoal}}$\;} \label{algo:inf_rrtstar:min}
        \newAlgoLine{$\xrand \gets \mathtt{Sample}\left(\xstart, \goalSet, \ccur \right)$\;} \label{algo:inf_rrtstar:sample}
        $\vnearest \gets \mathtt{Nearest}\left(\vertexSet, \xrand \right)$\;
        $\xnew \gets \mathtt{Steer}\left(\vnearest, \xrand\right)$\;
        \If{$\mathtt{IsFree}\left(\vnearest, \xnew\right)$}
        {
            \newAlgoLine
            {
                \If{ $\xnew \in \goalSet$ \label{algo:inf_rrtstar:goalStart}}
                {
                    $\solnVertices \setInsert \left\lbrace \xnew \right\rbrace$\;\label{algo:inf_rrtstar:goalEnd}
                }
            }
            $\vertexSet \setInsert \left\lbrace \xnew \right\rbrace$\;
            $\nearVertices \gets \mathtt{Near}\left(\vertexSet, \xnew, \rrewire \right)$\;
            $\vmin \gets \vnearest$\;
            \ForAll{$\vnear \in \nearVertices$}
            {
                $\cnew \gets \gAbove{\vnear} + \cTrue{\vnear}{\xnew}$\;
                \If{$\cnew < \gAbove{\vmin} + \cTrue{\vmin}{\xnew}$}
                {
                    \If{$\mathtt{IsFree}\left( \vnear, \xnew \right)$}
                    {
                        $\vmin \gets \vnear$\;
                    }
                }
            }
            $\edgeSet \setInsert \left\lbrace\left(\vmin, \xnew \right)\right\rbrace$\;
            
            \ForAll{$\vnear \in \nearVertices$}
            {
                $\cnear \gets \gAbove{\xnew} + \cTrue{\xnew}{\vnear}$\;
                \If{$\cnear < \gAbove{\vnear}$}
                {
                    \If{$\mathtt{IsFree}\left( \xnew, \vnear \right)$}
                    {
                        $\vparent \gets \mathtt{Parent}\left(\vnear \right)$\;
                        $\edgeSet \setRemove \left\lbrace \left(\vparent, \vnear \right) \right\rbrace$\;
                        $\edgeSet \setInsert \left\lbrace \left( \xnew, \vnear\right)\right\rbrace$\;
                    }
                }
            }
            \newAlgoLine
            {
                $\mathtt{Prune}\left(\vertexSet, \edgeSet, \ccur\right)$\label{algo:inf_rrtstar:prune}\;
            }
        }
    }
    \Return{$\treeGraph$}\;
\end{algorithm}%
\subsection{Notation}\label{sec:inf:note}
The tree, $\treeGraph \coloneqq \pair{\vertexSet}{\edgeSet}$, is defined by a set of vertices, $\vertexSet\subset\freeSet$, and edges, $\edgeSet = \set{\pair{\statev}{\statew}}$, for some $\statev,\, \statew \in \vertexSet$.
The function $\gAbove{\statev}$ represents the cost to reach a vertex, $\statev \in \vertexSet$, from the start given the current tree (the cost-to-come).
The function $\cTrue{\statev}{\statew}$ represents the cost of a path connecting the states $\statev, \statew \in \freeSet$, and corresponds to the edge cost between those two states if they are connected as vertices in the tree.
The notation $\stateSet \setInsert \set{\statex}$ and $\stateSet \setRemove \set{\statex}$ is used to compactly represent the compounding set operations $\stateSet \gets \stateSet \cup \set{\statex}$ and $\stateSet \gets \stateSet \setminus \set{\statex}$, respectively.
As is customary, the minimum of an empty set is taken to be infinity and a prolate hyperspheroid defined by an infinite transverse diameter is taken to have infinite measure.

\subsection{Graph Pruning (\algo{algo:prune})}\label{sec:inf:prune}
Graph pruning simplifies a tree by removing unnecessary vertices.
Vertices are often removed if their heuristic values are larger than the current solution (i.e., they do not belong to informed set).
While this identifies vertices that cannot provide a better solution, it is not a \emph{sufficient} condition to remove them without negatively affecting the search.
Their descendants may still be capable of providing better solutions (i.e., they may belong to the informed set; Fig.~\ref{fig:prune:defn}) in which case their removal would negatively affect performance by decreasing vertex density in the search domain (i.e., the informed set; Fig.~\ref{fig:prune:exp}b).

An \textit{admissible} pruning method that does not remove vertices from the informed set is presented in \algo{algo:prune}.
It iteratively removes leaves of the tree that cannot provide a better solution until no such leaves exist.
This only removes vertices if they \emph{and all their descendants} cannot belong to a better solution (i.e., it only removes vertices from outside the informed set; Fig.~\ref{fig:prune:defn}).
This retains all possibly beneficial vertices regardless of their current connections and does not alter the vertex distribution in areas being searched (Fig.~\ref{fig:prune:exp}c).

\begin{algorithm}[tbp]
    \algorithmStyle
    \caption{$\mathtt{Prune}\left(\vertexSet \subseteq \stateSet,\; \edgeSet \subseteq \vertexSet\times\vertexSet,\; \ccur\in\positiveReal\right)$}\label{algo:prune}
    \Repeat{$\pruneVertices = \emptyset$}
    {
        $\pruneVertices \gets \setst{\statev\in\vertexSet}{\fBelow{\statev} > \ccur,\;\; \mathtt{and}\;\; \forall \statew\in\vertexSet,\; \pair{\statev}{\statew} \not\in\edgeSet}$\;
        $\edgeSet \setRemove \setst{\pair{\stateu}{\statev}\in\edgeSet}{\statev\in\pruneVertices}$\;
        $\vertexSet \setRemove \pruneVertices$\;
    }
\end{algorithm}%
\begin{figure}[tbp]
    \centering
    \includegraphics[scale=1]{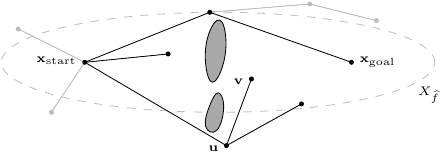} %
    \caption{%
    An illustration of \algo{algo:prune} that shows the retained (black) and pruned (grey) vertices given the $L^2$ informed set (dashed grey line) defined by the current solution.
    Vertices are pruned if and only if they cannot improve the current solution (i.e., they are not members of the $L^2$ informed set) \emph{and} neither can their descendants.
    This pruning condition avoids removing promising vertices (e.g., $\statev$) simply because they are currently descendants of vertices outside the subset (e.g., $\stateu$) and maintains the vertex distribution of the $L^2$ informed set (Fig.~\ref{fig:prune:exp}).
    }
    \label{fig:prune:defn}
\end{figure}%
\subsection{The Rewiring Neighbourhood}\label{sec:inf:rewire}
\ac{RRTstar} almost-surely converges asymptotically to the optimum by incrementally rewiring the tree around new states.
In the $\radius$-disc variant this is the set of states within a radius, $\rrewire$, of the new state,
\begin{equation}\label{eqn:back:rewire}
    \rrewire \coloneqq \min\set{\maxEdge, \rrrtstar},
\end{equation}
where $\maxEdge$ is the maximum allowable edge length of the tree and $\rrrtstar$ is a function of the problem measure and the number of vertices in the tree \cite{karaman_ijrr11}.
Specifically,
\begin{align}\label{eqn:back:radius}
    \rrrtstar &>  \rrrtstarmin,\nonumber\\
    \rrrtstarmin&\coloneqq \left(2 \left(1 + \frac{1}{\dimension}\right)
    \left(\frac{\lebesgue{\stateSet}}{\unitBall{\dimension}}\right)
    \left(\frac{\log\left(\card{\vertexSet}\right)}{\card{\vertexSet}}\right)\right)^{\frac{1}{\dimension}},
\end{align}
where $\lebesgue{\cdot}$ is the Lebesgue measure of a set (e.g., the \textit{volume}), $\unitBall{\dimension}$ is the Lebesgue measure of an $\dimension$-dimensional unit ball, i.e., \eqref{eqn:ballMeasure}, and $\card{\cdot}$ is the cardinality of a set.

The rewiring neighbourhood in the $k$-nearest variant is the $\krrtstar$-closest states to the new state, where
\begin{align}\label{eqn:back:knearest}
    \krrtstar &>  \krrtstarmin,\nonumber\\
    \krrtstarmin &\coloneqq e\left(1+\frac{1}{\dimension}\right)\log\left(\card{\vertexSet}\right).
\end{align}

Informed \ac{RRTstar} searches a subset of the original planning problem.
The rewiring requirements to maintain almost-sure asymptotic optimality in this shrinking domain will be a function of the number of vertices in the informed set, $\card{\vertexSet\cap\fBelowSet}$, and its measure, $\lebesgue{\fBelowSet}$.
The $L^2$ informed set is not known in closed form (it is an intersection of a prolate hyperspheroid and free space) but its measure can be bounded from above by the minimum measure of the prolate hyperspheroid and the problem domain,
\begin{equation*}
    \lebesgue{\fBelowSet}\leq\min\left\lbrace\lebesgue{\stateSet},\lebesgue{\phsSet}\right\rbrace.
\end{equation*}

This updates \eqref{eqn:back:radius} and \eqref{eqn:back:knearest} to
\begin{align}\label{eqn:radius}
    \rrrtstarmin&\leq \left(2 \left(1 + \frac{1}{\dimension}\right)
    \left(\frac{\min\left\lbrace\lebesgue{\stateSet},\lebesgue{\phsSet}\right\rbrace}{\unitBall{\dimension}}\right)
        \vphantom{\left(\frac{\log\left(\card{\vertexSet\cap\fBelowSet}\right)}{\card{\vertexSet\cap\fBelowSet}}\right)}
        \right.\nonumber\\&\qquad\qquad\qquad\qquad\;\;\left.
        \vphantom{2 \left(1 + \frac{1}{\dimension}\right)
        \left(\frac{\min\left\lbrace\lebesgue{\stateSet},\lebesgue{\phsSet}\right\rbrace}{\unitBall{\dimension}}\right)
        \left(\frac{\log\left(\card{\vertexSet\cap\fBelowSet}\right)}{\card{\vertexSet\cap\fBelowSet}}\right)}
    \left(\frac{\log\left(\card{\vertexSet\cap\fBelowSet}\right)}{\card{\vertexSet\cap\fBelowSet}}\right)\right)^{\frac{1}{\dimension}}
\end{align}
and
\begin{equation}\label{eqn:knearest}
    \krrtstarmin = e\left(1+\frac{1}{\dimension}\right)\log\left(\card{\vertexSet\cap\fBelowSet}\right),
\end{equation}
where $\lebesgue{\phsSet}$ is a function of the current solution, i.e., \eqref{eqn:phsMeasure}.
\begin{figure}[tbp]
    \centering
    \includegraphics[page=1,width=\columnwidth]{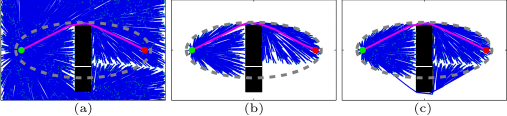}
    \caption{%
        An illustration of pruning a graph found by \ac{RRTstar}, (a), with both inadmissible, (b), and admissible, (c), methods.
        By removing all vertices that cannot belong to a better solution, the inadmissible method may greedily remove descendent vertices that will later provide a better solution once the graph is improved.
        By only removing vertices that cannot improve a solution if neither can their descendants, the admissible method (\algo{algo:prune}) maintains a uniform sample density in the entire informed set.
    }
    \label{fig:prune:exp}
\end{figure}%

These rewiring neighbourhoods will be smaller than \eqref{eqn:back:radius} and \eqref{eqn:back:knearest} when they can contain fewer vertices (i.e., only those in the informed set) and/or a smaller problem measure (i.e., the measure of the informed set).
Smaller rewiring neighbourhoods reduce the computational cost of rewiring at each iteration and improves the real-time performance of Informed \acs{RRTstar} while maintaining almost-sure asymptotic optimality.

\section{Rates of Convergence}\label{sec:rate}
Almost-sure asymptotic optimality provides no insight into the rate at which solutions are improved.
Previous work has found probabilistic rates for \ac{PRMstar} \cite{dobson_icra15} and \ac{FMTstar} \cite{janson_ijrr15} and estimated the expected length of \ac{RRTstar} solutions as a function of computational time \cite{dobson_icra15}.

Performance can be quantified analytically by evaluating the rate at which the sequence of solution costs converges to the optimum.
This rate can be classified as sublinear, linear, or superlinear (Definition~\ref{defn:conv}).
\begin{defn}[Rate of convergence]\label{defn:conv}
    A sequence of numbers, $\seqidx{a}{\iter}{1}{\infty}$, that monotonically and asymptotically approaches a limit, $a_\infty$, has a rate of convergence given by
    \begin{equation*}
        \rate \coloneqq \limitoinf\frac{\left|a_{\iter+1}-a_{\infty}\right|}{\left|a_{\iter}-a_{\infty}\right|}.
    \end{equation*}
    The sequence is said to converge \emph{linearly} if the rate is in the range $0 < \rate < 1$, \emph{superlinearly} (i.e., faster than linear) when $\rate=0$, and \emph{sublinearly} (i.e., slower than linear) when $\rate=1$.
\end{defn}

The expected convergence rate of an algorithm depends on its tuning and the planning problem.
General rates can be calculated for holonomic minimum-path-length problems for \ac{RRTstar} with and without sample rejection and Informed \acs{RRTstar} (Theorems~\ref{thm:conv:rrtstar}--\ref{thm:conv:informed}) by first calculating sharp bounds on the expected next-iteration cost (Lemma~\ref{lem:conv:expect}).

\begin{lem}[Expected next-iteration cost of minimum-path-length planning]\label{lem:conv:expect}
    The expected value of the next solution to a minimum-path-length problem, $\expect{\cnext}$, is bounded by
    \begin{equation}\label{eqn:lem:conv:expect}
        \probBetter\frac{\dimension\ccur^2 + \cmin^2}{\left(\dimension + 1\right)\ccur} + \left(1-\probBetter\right)\ccur \leq \expect{\cnext} \leq \ccur,
    \end{equation}
    where $\ccur$ is the current solution cost, $\cmin$ is the theoretical minimum solution cost, $\dimension$ is the state dimension of the planning problem, and $\probBetter = \prob{\xnew \in \fTrueSet}$ is the probability of adding a state that is a member of the omniscient set (i.e., that can belong to a better solution).
    While not explicitly shown, the subset, $\fTrueSet$, and the probability of improving the solution, $\probBetter$, are generally functions of the current solution cost.
    
    This lower bound is sharp over the set of all possible minimum-path-length planning problems and algorithm configurations and is exact for versions of \ac{RRTstar} with an infinite rewiring radius (i.e., $\maxEdge=\infty$, and $\rrrtstar=\infty$) searching an obstacle-free environment without constraints.
\end{lem}%
\begin{proof}
    The proof of Lemma~\ref{lem:conv:expect} from the supplementary online material appears in Appendix~\ref{appx:infinite:expect}.
\end{proof}

This result allows sharp bounds on the convergence rates of \ac{RRTstar} (with and without rejection sampling) and Informed \acs{RRTstar} to be calculated for any configuration or holonomic minimum-path-length planning problem.
These bounds will be exact in problems without obstacles and constraints and with an infinite rewiring neighbourhood (i.e., $\maxEdge=\infty$, and $\rrrtstar=\infty$) and show that \ac{RRTstar} \emph{always} has sublinear convergence to the optimum (Theorem~\ref{thm:conv:rrtstar}).

\begin{thm}[Sublinear convergence of \acs{RRTstar} in holonomic minimum-path-length planning]\label{thm:conv:rrtstar}
    \ac{RRTstar} converges \emph{sublinearly} towards the optimum of holonomic minimum-path-length planning problems,
    \begin{equation}\label{eqn:thm:conv:rrtstar}
        \expect{\rrtstarRate} = 1.
    \end{equation}
    
    For simplicity, this statement is limited to holonomic planning but it can be extended to specific constraints by expanding Lemma~\ref{lem:necessary:exact:sample}.
\end{thm}
\begin{proof}
    The proof of Theorem~\ref{thm:conv:rrtstar} follows directly from Lemma~\ref{lem:conv:expect} when $\probBetter$ is given by \eqref{eqn:sampleProb} and appears from the supplementary online material in Appendix~\ref{appx:infinite:rate:rrtstar}.
    \squeezeWidowedWords
\end{proof}

Rectangular rejection sampling improves the convergence rate of \ac{RRTstar}.
This improvement is maximized by sampling a rectangle that tightly bounds the informed set (Fig.~\ref{fig:sampleTheory}a).
The resulting adaptive rectangular rejection sampling (e.g., \cite{otte_tro13}) allows \ac{RRTstar} to converge linearly in the absence of obstacles and constraints and with an infinite rewiring neighbourhood (Theorem~\ref{thm:conv:reject}).

\begin{thm}[Linear convergence of \acs{RRTstar} with adaptive rectangular rejection sampling in holonomic minimum-path-length planning]\label{thm:conv:reject}
    \ac{RRTstar} with adaptive rectangular rejection sampling converges at best \emph{linearly} towards the optimum of holonomic minimum-path-length planning problems but factorially approaches sublinear convergence with increasing state dimension,
    \begin{equation}\label{eqn:thm:conv:reject}
        1 - \frac{\pi^{\frac{\dimension}{2}}}{\left(\dimension+1\right)2^{\dimension-1}\gammaFunc{\frac{\dimension}{2}+1}} \leq \expect{\rejectRate} \leq 1.
    \end{equation}
    
    For simplicity, this statement is limited to holonomic planning but it can be extended to specific constraints by expanding Lemma~\ref{lem:necessary:exact:sample}.
\end{thm}
\begin{proof}
    The proof of Theorem~\ref{thm:conv:reject} follows directly from Lemma~\ref{lem:conv:expect} when $\probBetter$ is calculated by substituting \eqref{eqn:thm:curse:tightMeasure} in \eqref{eqn:sampleProb} and appears from the supplementary online material in Appendix~\ref{appx:infinite:rate:reject}.
\end{proof}

This convergence rate diminishes factorially (i.e., quickly) as state dimension increases due to the minimum-path-length curse of dimensionality.
Informed \acs{RRTstar} avoids this limitation with direct informed sampling.
It also converges linearly in the absence of obstacles and constraints and with an infinite rewiring neighbourhood but has a weaker dependence on state dimension (Theorem~\ref{thm:conv:informed}).

\begin{thm}[Linear convergence of Informed \acs{RRTstar} in holonomic minimum-path-length planning]\label{thm:conv:informed}
    Informed \acs{RRTstar} converges at best \emph{linearly} towards the optimum of holonomic minimum-path-length planning problems,
    \begin{equation}\label{eqn:thm:conv:informed}
         \frac{n-1}{n+1} \leq \expect{\informedRate} \leq 1,
    \end{equation}
    where the lower-bound occurs exactly with an infinite rewiring neighbourhood in the absence of obstacles and constraints.
    
    For simplicity, this statement is limited to holonomic planning but it can be extended to specific constraints by expanding Lemma~\ref{lem:necessary:exact:sample}.
\end{thm}
\begin{proof}
    The proof of Theorem~\ref{thm:conv:informed} follows directly from Lemma~\ref{lem:conv:expect} when $\probBetter=1$ and appears from the supplementary online material in Appendix~\ref{appx:infinite:rate:informed}.
\end{proof}

\begin{figure}[tbp]
    \centering
    \includegraphics[width=\columnwidth]{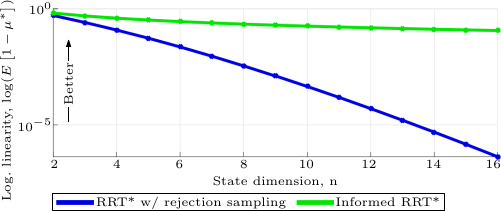}
    \caption{%
            An illustration of the lower-bounds on linearity, $\expect{1 - \rate^*}$, of \acs{RRTstar} with rejection sampling and Informed \acs{RRTstar} (Corollary~\ref{cor:conv:always_better}).
            As predicted by Theorems~\ref{thm:conv:reject} and \ref{thm:conv:informed}, the convergence rates bounds diverge as state dimensions increase, with rejection sampling factorially approaching sublinear convergence.
            }
    \label{fig:conv:rates}
\end{figure}%
\begin{figure*}[tbp]
    \centering
    \includegraphics[width=\textwidth]{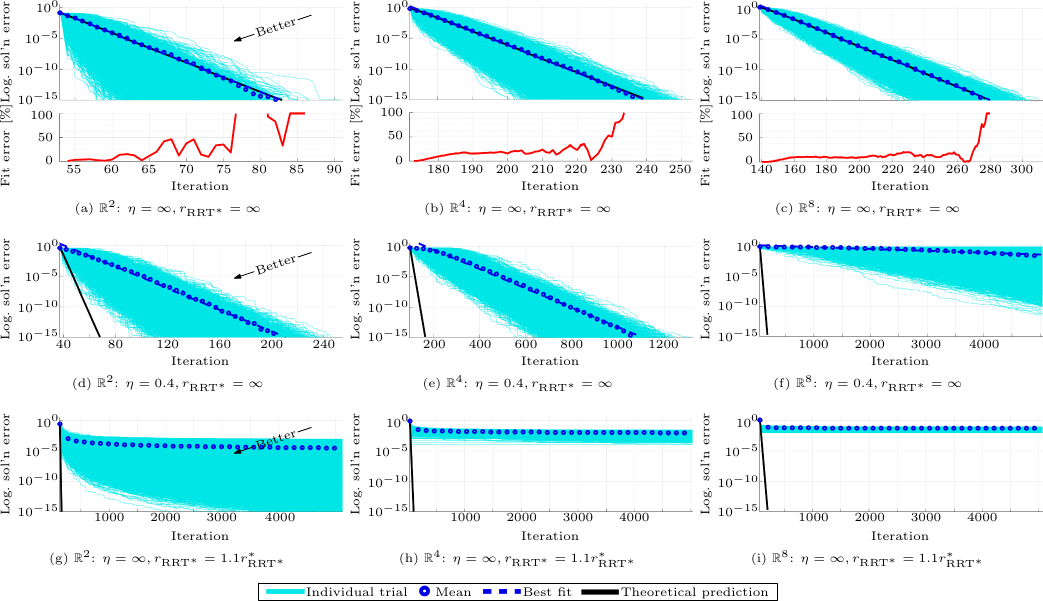}
    \caption{%
            Experimental validation and extension of Lemma~\ref{lem:conv:expect} and Theorem~\ref{thm:conv:informed} in $\real^2$, $\real^4$ and $\real^8$.
            Informed \acs{RRTstar} was run from a common initial solution $10^4$ times in $\real^2$, $\real^4$ and $\real^8$ with different pseudo-random seeds.
            The error relative to the known optimum, $\log\left(\ccur - \copt\right)$, was plotted for each instance at each iteration (cyan lines) along with the mean error (blue circles), a line of best fit (blue dashed line), and the lower-bound error predicted by Lemma~\ref{lem:conv:expect} (black line).
            The difference between the predicted lower bound and the mean errors (red lines), $\left|\left(\cost_{\mathrm{mean},i}-\cost_{\mathrm{theory},i}\right)/\left(\cost_{\mathrm{mean},i}-\copt\right)\right|$, with infinite rewiring neighbourhoods, (a)--(c), confirms experimentally that convergence is linear (Theorem~\ref{thm:conv:informed}).
            The mean error for a finite but \emph{constant} rewiring neighbourhood, (d)--(f), shows experimentally that convergence is slower but possibly still linear.
            The mean error for a finite and \emph{decreasing} rewiring neighbourhood, (g)--(i), shows experimentally that the is slower and sublinear.
            The results of (d)--(i) motivate further research on the effects of the \acs{RRTstar} rewiring neighbourhood.
    }
    \label{fig:conv:exp:combo}
\end{figure*}
Theorems~\ref{thm:conv:rrtstar}--\ref{thm:conv:informed} result in the following corollary regarding the relative convergence rates of the algorithms.
\begin{cor}[The faster convergence of Informed \acs{RRTstar} in holonomic minimum-path-length planning]\label{cor:conv:always_better}
    The best-case convergence rate of Informed \acs{RRTstar}, $\informedRate^*$, is always better than that of \ac{RRTstar}, with or without rejection sampling in holonomic minimum-path-length planning,
    \begin{equation*}
        \forall n \geq 2,\; \frac{n-1}{n+1} = \expect{\informedRate^*} \leq \expect{\rejectRate^*} \leq \expect{\rrtstarRate^*} = 1.
    \end{equation*}
    
    For simplicity, this statement is limited to holonomic planning but it can be extended to specific constraints by expanding Lemma~\ref{lem:necessary:exact:sample}.
\end{cor}
\begin{proof}
    The proof follows immediately from the lower bounds in \eqref{eqn:thm:conv:rrtstar}, \eqref{eqn:thm:conv:reject}, and \eqref{eqn:thm:conv:informed}.
    It is illustrated in Fig.~\ref{fig:conv:rates}.
\end{proof}

\subsection{Experimental Validation and Extension}\label{sec:rate:exp}
Convergence rates are investigated experimentally for infinite, constant finite, and decreasing finite rewiring radii.
To isolate the effects of the rewiring parameters, Informed \acs{RRTstar} was run on obstacle-{} and constraint-free problems in $\real^2$, $\real^4$, and $\real^8$ for $10^4$ trials of each configuration.
Each trial started from the same initial solution but used different pseudo-random seeds to search for improvements.
The logarithmic error relative to the known optimum, $\log\left(\ccur - \copt\right)$, and the resulting mean were calculated at each iteration of each trial and used to validate Theorem~\ref{thm:conv:informed} and illustrate the effects of rewiring parameters on the convergence rate.

The experimental results for an infinite rewiring neighbourhood (i.e., $\maxEdge=\infty$ and $\rrrtstar=\infty$) show excellent agreement with the theoretical predictions in Theorem~\ref{thm:conv:informed} (Figs.~\ref{fig:conv:exp:combo}a--c).
The mean solution cost converges linearly towards the optimum and closely matches the lower-bound predicted by Lemma~\ref{lem:conv:expect}.

The experimental results for a \emph{constant finite} rewiring neighbourhood (i.e., $\eta = 0.4$ and $\rrrtstar=\infty$) show that the convergence rate is lower than predicted by Theorem~\ref{thm:conv:informed} (Figs.~\ref{fig:conv:exp:combo}d--f).
The convergence rate appears to be initially nonlinear but then become linear.
It is hypothesized that this is related to the density of samples relative to the maximum edge length as reflected by $\iterThresh$ in Theorem~\ref{thm:necessary:exact:prob}.

The experimental results for a \emph{decreasing finite} rewiring neighbourhood (i.e., $\eta = \infty$ and $\rrrtstar=1.1\rrrtstarmin$) show that the convergence rate appears to be sublinear (Figs.~\ref{fig:conv:exp:combo}g--i).
It is hypothesized that this is a result of the rewiring neighbourhood shrinking `too' fast relative to the sample density.

These experiments suggest that further research is necessary to study the tradeoff between per-iteration cost and the number of iterations needed to find a solution.
While a shrinking rewiring neighbourhood limits the number of rewirings, the apparent resulting sublinear convergence would require significantly more iterations to find high-quality solutions.
Alternatively, while linear convergence needs fewer iterations to find equivalent solutions, the required constant radius would allow the number of rewirings to increase indefinitely.

\section{Experiments}\label{sec:exp}
Informed \acs{RRTstar} was evaluated on simulated problems in $\real^2$, $\real^4$, and $\real^8$ (Sections~\ref{sec:exp:toy} and \ref{sec:exp:grid}) and for \ac{HERB} (Section~\ref{sec:exp:herb}) using \ac{OMPL}\footnotemark{}.
\footnotetext{The experiments were run on a laptop with $16$~GB of RAM and an Intel i7-4810MQ processor. The abstract experiments were run in Ubuntu 12.04 ($64$-bit) with Boost 1.58, while the \acs{HERB} experiments were run in Ubuntu 14.04 ($64$-bit).}
It was compared to the original \ac{RRTstar} and versions that focus the search with graph pruning (e.g., \algo{algo:prune}), heuristic rejection on $\xnew$, heuristic rejection on $\xrand$, and all three techniques combined.

All planners used the same tuning parameters and the ordered rewiring technique presented in \cite{perez_iros11}.
Planners used a goal-sampling bias of $5\%$ and an \ac{RRTstar} radius of $\rrrtstar=2\rrrtstarmin$.
The maximum edge length was selected experimentally to reduce the time required to find an initial solution on a training problem, with values of $\maxEdge=0.3$, $0.5$, $0.9$, and $1.3$ used in $\real^2$, $\real^4$, $\real^8$, and on \ac{HERB} ($\real^{14}$), respectively.
Available planning time was limited for each state dimension to $3$, $30$, $150$, and $600$~seconds, respectively.
Planners with heuristics used the $L^2$ norm as estimates of cost-to-come and cost-to-go while those with graph pruning delayed its application until solution cost changed by more than $5\%$.

These experiments were designed to investigate admissible methods to focus search.
More advanced extensions of \ac{RRTstar} were not considered as they commonly include some combination of the investigated techniques. %

\subsection{Toy Problems}\label{sec:exp:toy}
Two separate experiments were run in $\real^2$, $\real^4$, and $\real^8$ on randomized variants of the toy problem depicted in Fig.~\ref{fig:exp:defn}a to investigate the effects of obstacles on convergence.

The problem consists of a (hyper)cube of width $\mapwidth$ with a single start and goal located at $\left[-0.5,0,\ldots,0\right]^T$ and $\left[0.5,0,\ldots,0\right]^T$, respectively.
A single (hyper)cube obstacle of width $\obswidth\sim\uniformSymb\left[0.25,0.5\right]$ sits between the start and goal in the centre of the problem domain.

The first experiment investigates finding near-optimal solutions in the presence of obstacles.
The time required for each planner to find a solution within various fractions of the known optimum, $\copt$, was recorded over $100$ trials with different pseudo-random seeds for maps of width $\mapwidth=2$.
The percentage of trials that found a solution within the target tolerance of the optimum and the median time necessary to do so are presented for each planner in Figs.~\ref{fig:exp}a--c.
Trials that did not find a suitable solution were treated as having infinite time for the purpose of calculating the median.
The results show that Informed \acs{RRTstar} performs equivalently to rejection sampling algorithms in low state dimensions but outperforms all existing techniques in higher dimensions.

The second experiment investigates finding near-optimal solutions in large planning problems.
The time required for each planner to find a near-optimal solution was recorded over $100$ trials with different pseudo-random seeds for maps of increasing width, $\mapwidth$.
Planners sought a solution better than $1.01\copt$, $1.05\copt$, and $1.15\copt$ in $\real^2$, $\real^4$, and $\real^8$, respectively.
The percentage of trials that found a sufficiently near-optimal solution and the median time necessary to do so are presented for each planner in Figs.~\ref{fig:exp}d--f.
Trials that did not find a suitable solution were treated as having infinite time for the purpose of calculating the median.
The results show that Informed \acs{RRTstar} outperforms all existing techniques in large-domain planning problems and that the difference increases in higher state dimensions.

These experiments show that increasing problem size and state dimension decreases the ability of nondirect sampling methods to find near-optimal solutions, as predicted by \eqref{eqn:betterProb}.
Informed \acs{RRTstar} limits these effects and outperforms existing techniques by efficiently focusing its search to the $L^2$ informed set using direct informed sampling.

\subsection{Worlds with Many Homotopy Classes}\label{sec:exp:grid}
The algorithms were tested on more complicated problems with many homotopy classes in $\real^2$, $\real^4$, and $\real^8$.
The worlds consisted of a (hyper)cube of width $l=4$ with the start and goal located at $\left[-0.5,0,\ldots,0\right]^T$ and $\left[0.5,0,\ldots,0\right]^T$, respectively.
The problem domain was filled with a regular pattern of axis-aligned (hyper)cube obstacles with a width such that the start and goal were $5$ `columns' apart (Fig.~\ref{fig:exp:defn}b).

The planners were tested with $100$ different pseudo-random seeds on each world and state dimension.
The solution cost of each planner was recorded every $1$~millisecond by a separate thread and the median was calculated from the $100$ trials by interpolating each trial at a period of $1$~millisecond.
The absence of a solution was considered an infinite cost for the purpose of calculating the median.

The results are presented in Figs.~\ref{fig:exp}g--i, where the percent of trials solved and the median solution cost are plotted versus run time.
They demonstrate how Informed \acs{RRTstar} has better real-time convergence towards the optimum than existing techniques, especially in higher state dimensions.

\begin{figure}[tb]
    \centering
    \includegraphics[width=\columnwidth]{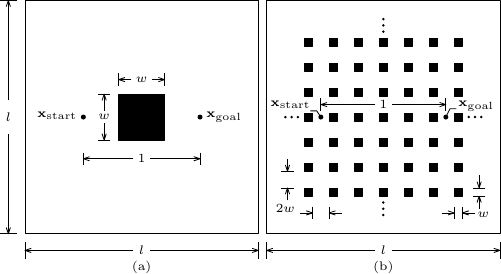}
    \caption{%
            Illustrations of the planning problems used for Sections~\ref{sec:exp:toy} and \ref{sec:exp:grid} to study performance relative to a known optimum, the effect of map width, $l$, and performance in problems with many homotopy classes.
            The width of the obstacle in (a) is a random variable uniformly distributed over the range $\left[0.25,0.5\right]$.
            The regularly spaced obstacles in (b) are chosen in to scale efficiently to high dimensions and their width is such that the start and goal states are $5$ `columns' apart.
    }
    \label{fig:exp:defn}
\end{figure}%
\begin{figure*}[tbp]
    \centering
    \includegraphics[width=\textwidth]{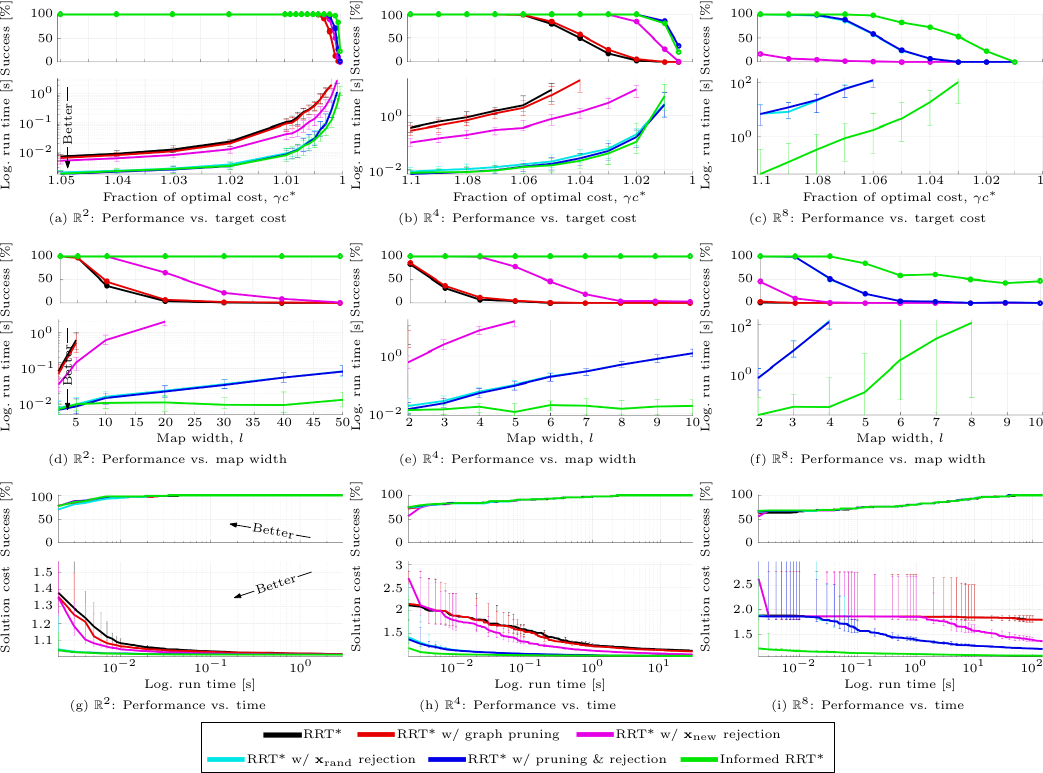}
    \caption{%
            Results for the experiments described in Sections~\ref{sec:exp:toy} and \ref{sec:exp:grid}.
            Each planner was run $100$ different times in $\real^2$, $\real^4$, and $\real^8$ on each problem for $3$, $30$, and $150$ seconds, respectively. 
            The percentage of trials that found the desired solution are plotted above and the median performance is plotted below for each experiment.
            Unsuccessful trials were assigned an infinite value for the purpose of calculating the median and the error bars denote a nonparamentric $99\%$ confidence interval on the median value.
            The times required to find different near-optimal solutions, $\ccur<\gamma\copt$, for the problem illustrated in Fig.~\ref{fig:exp:defn}a with $l=2$ are presented in (a)--(c).
            The times required to find a solution within a fraction of the known optimum ($1.01\copt$, $1.05\copt$, and $1.15\copt$, respectively) for the problem illustrated in Fig.~\ref{fig:exp:defn}a for various map widths are presented in (d)--(f).
            Solution cost is plotted versus run time for the problem illustrated in Fig.~\ref{fig:exp:defn}b in (g)--(i).
            Taken together, these experiments demonstrate the benefits of direct informed sampling even in large or high-dimensional problems, with a high number of obstacles, and many homotopy classes.
            }%
    \label{fig:exp}
\end{figure*}

\subsection{Motion Planning for \acs{HERB}}\label{sec:exp:herb}
Informed \ac{RRTstar} was demonstrated on a high-dimensional problem using \ac{HERB}, a 14-\ac{DOF} mobile manipulation platform \cite{herb}.
Poses were defined for the two arms to create a sequence of three planning problems (Fig.~\ref{fig:exp:herb}) inspired by \cite{ymca}.
The objective of these problems was to find the minimum path length through a $14$-dimensional search space with strict limits (each joint has no more than $\pi$--$2\pi$ radians of travel).
While path length is not a common cost function for manipulation, these experiments illustrate that direct informed sampling is beneficial in high-dimensional problem domains even with strict search limits.

\ac{RRTstar}, \acs{RRTstar} with pruning and rejection, and Informed \acs{RRTstar} were each run for $50$ trials on each problem of the cycle.
The resulting median path lengths are presented in Fig.~\ref{fig:exp:herb:bars}.
Trials that did not find a solution were considered to have infinite length for the purpose of calculating the median.
This only occurred for the problem from (a) to (b), where the planners found a solution on $94\%$ of the trials.

\ac{RRTstar} with and without pruning and rejection sampling both fail to improve the initial solutions on all three planning problems but Informed \acs{RRTstar} is able to improve the path length by $3.9\%$, $7.9\%$, and $28.2\%$, respectively.
The improvement for (a) to (b) is not statistically significant but (b) to (c) and (c) to (d) demonstrate the benefits of considering the relative sizes of the informed set and problem domain in high state dimensions.

\section{Discussion \& Conclusion}\label{sec:fin}
\ac{RRTstar} almost-surely converges asymptotically to the optimum by asymptotically finding the optimal paths to every state in the problem domain.
This is inefficient in single-query scenarios as, once a solution is found, searches only need to consider states that can belong to a better solution (i.e., the omniscient set; Definition~\ref{defn:omni}, Lemma~\ref{lem:necessary:exact}).
Previous work has focused search to estimates of this set (i.e., informed sets; Definition~\ref{defn:informed}) but has not used these estimates to analyze performance.
This paper proves that for holonomic problems the probability of sampling an admissible informed set provides an upper bound on the probability of improving a solution (Theorem~\ref{thm:necessary:heuristic:prob}).
\squeezeWidowedWords

A popular admissible heuristic for problems seeking to minimize path length is the $L^2$ norm (i.e., Euclidean distance).
This paper shows that existing techniques to exploit it are insufficient.
The majority of approaches either reduce the ability to find solutions in other homotopy classes (i.e., reduce recall; Definition~\ref{defn:recall}) or fail to account for the reduction of the $L^2$ informed set in response to solution improvement (i.e., have decreasing precision; Definition~\ref{defn:precision}).
Even existing adaptive techniques that address these problems (e.g., \cite{otte_tro13}) fail to account for its factorial decrease in measure with state dimension (i.e., the minimum-path-length curse of dimensionality; Theorem~\ref{thm:curse}).

This paper presents a method to avoid these limitations through direct sampling of the $L^2$ informed set (\algos{algo:infset}{algo:randomKeep}; Section~\ref{sec:l2}).
This approach generates uniformly distributed samples in the informed set regardless of its size relative to the problem domain or the state dimension (i.e., it has $100\%$ recall and high precision).
This paper presents Informed \acs{RRTstar} as a demonstration of how these techniques can be used in sampling-based planning (\algoAnd{algo:inf_rrtstar}{algo:prune}; Section~\ref{sec:inf}).

\begin{figure}[tbp]
    \centering
    \includegraphics[width=\columnwidth]{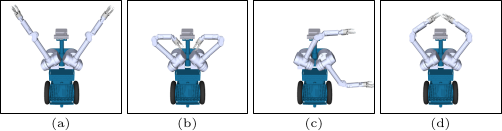}
    \caption{%
            A motion planning problem for \acs{HERB} inspired by \cite{ymca}.
            Planners must find a collision-free path between each pair of subsequent poses, e.g., (a) to (b).
            \acs{HERB}'s $14$ \acsp{DOF} and large number of potential self-collisions make this a nontrivial planning problem for \acs{RRTstar}.
            The planners were given $600$ seconds for each phase of the planning problem and the results are presented in Fig~\ref{fig:exp:herb:bars}.
            }
    \label{fig:exp:herb}
\end{figure}

Informed \acs{RRTstar} considers all homotopy classes that could provide a better solution (i.e., $100\%$ recall), unlike sample biasing techniques.
It is effective regardless of the relative size of the informed set or the state dimension, unlike sample rejection or graph pruning.
When the heuristic does not provide any information (e.g., small planning problems and/or large informed sets) it is identical to \ac{RRTstar}.

This paper also uses the shape of the $L^2$ informed set to analyze the theoretical performance of \ac{RRTstar} on minimum-path-length problems (Section~\ref{sec:rate}) by bounding the expected solution cost (Lemma~\ref{lem:conv:expect}) and convergence rates (Theorems~\ref{thm:conv:rrtstar}--\ref{thm:conv:informed}).
The bounds are sharp over the set of all (Lemma~\ref{lem:conv:expect}) or all holonomic (Theorems~\ref{thm:conv:rrtstar}--\ref{thm:conv:informed}) minimum-path-length planning problems and algorithm configurations with the lower bounds exact for an infinite rewiring radius in the absence of obstacles and constraints.
These results prove that \ac{RRTstar} converges sublinearly (i.e., slower than linear) for all configurations and holonomic minimum-path-length problems and that focused variants (e.g., Informed \acs{RRTstar}) can have linear convergence.

This analysis is extended experimentally to different configurations.
The results confirm the theoretical findings and suggest that obstacle-{} and constraint-free convergence remains linear when the rewiring radius is constant but becomes sublinear when it decreases in the manner proposed by \cite{karaman_ijrr11}.
As previous analysis of this radius has focused on per-iteration complexity, we believe this result motivates future research into the trade off between per-iteration cost and convergence rate.

The practical advantages of Informed \ac{RRTstar} are shown on a variety of planning problems (Section~\ref{sec:exp}).
These experiments demonstrate how its theoretical convergence rate corresponds to better performance on real planning problems.
The amount of improvement depends on how efficiently the $L^2$ informed set decreases the search domain and may be limited in small problem domains and/or long circuitous solutions (e.g., the small/low-dimensional problems in Section~\ref{sec:exp:toy} and the first problem of Section~\ref{sec:exp:herb}).
The design of \algo{algo:multigoal} assures that in these situations Informed \ac{RRTstar} performs no worse than other methods to exploit the $L^2$ heuristic (e.g., rejection sampling).

Designing these experiments highlighted the relationship between the maximum edge length, $\maxEdge$, and algorithm performance.
This user-selected value not only affected the time required to find an initial solution but, as a result of \eqref{eqn:back:rewire}, also the quality of the solution found in finite time.
Specifically, large values of $\maxEdge$ appeared to decrease the difference between algorithms; however, also resulted in order of magnitude increases in the time required to find initial solutions.
When coupled with the results of Section~\ref{sec:exp}, this result should further motivate more research into the effects of the \ac{RRTstar} tuning parameters, $\maxEdge$ and $\rrrtstar$, on real-time performance.
Given that anytime improvement of a solution is a major feature of \ac{RRTstar}, we tuned $\maxEdge$ for these experiments to minimize the initial-solution time on a series of independent test problems.

We believe that defining precise and admissible informed sets is a fundamental challenge of using anytime almost-surely asymptotically optimal planners in real-world applications.
The $L^2$ informed set is a sharp, uniformly admissible estimate of the omniscient set for problems seeking to minimize path length, even in the presence of constraints, and is exact in the absence of obstacles and constraints.
This suggests that any informed set that is more precise must either
\begin{inparaenum}[(i)]
    \item exploit additional information about the problem domain (e.g., obstacles, constraints), and/or
    \item be inadmissible for some minimum-path-length planning problems.
\end{inparaenum}
Finding ways to define new admissible heuristics from additional problem-specific information could potentially allow focused search algorithms to converge linearly in the presence of obstacles and/or constraints.

We ultimately believe that heuristics are a key component of successful planning algorithms.
To this end, we are currently investigating methods to extend heuristics to entire sampling-based searches, similar to how A* \cite{hart_tssc68} extends Dijkstra's algorithm \cite{dijkstra_59}.
We accomplish this in \ac{BITstar} \cite{gammell_icra15,gammell_phd17,gammell_ijrr18} by extending the ideas presented in this paper to batches of randomly generated samples.
These samples are limited to informed sets and searched in order of potential solution quality.
Information on \ac{OMPL} implementations of both Informed \acs{RRTstar} and \ac{BITstar} are available at \href{http://asrl.utias.utoronto.ca/code}{\footnotesize\url{http://asrl.utias.utoronto.ca/code}}.
\begin{figure}[tbp]
    \centering
    \includegraphics[width=\columnwidth]{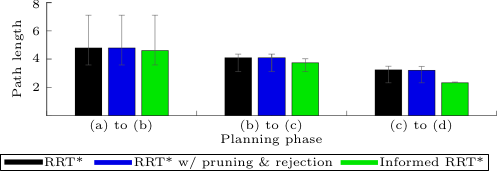}
    \caption{%
            Median path length results from the motion planning problems depicted in Fig.~\ref{fig:exp:herb}.
            Planners found a solution between each pose in every trial after $600$ seconds other than the transition from (a) to (b), where solutions were only found in $94\%$ of the $50$ trials.
            For the purpose of calculating the median, these unsolved trials were assigned an infinite cost.
            Error bars denote a nonparamentric $99\%$ confidence interval on the median value.
            The results show that even in the presence of strict state-space limits, Informed \acs{RRTstar} can outperform rejection sampling in high-dimensional problems.
            }
    \label{fig:exp:herb:bars}
\end{figure}
\theoremstyle{plain}%
\newtheorem*{lem:exact}{Lemma~\ref{lem:necessary:exact}}%
\newtheorem*{lem:sample}{Lemma~\ref{lem:necessary:exact:sample}}%
\newtheorem*{lem:uniform}{Lemma~\ref{lem:uniform}}%
\newtheorem*{lem:expect}{Lemma~\ref{lem:conv:expect}}%
\newtheorem*{thm:conv_rrtstar}{Theorem~\ref{thm:conv:rrtstar}}%
\newtheorem*{thm:conv_reject}{Theorem~\ref{thm:conv:reject}}%
\newtheorem*{thm:conv_informed}{Theorem~\ref{thm:conv:informed}}%

\appendices
\section{Proofs of Lemmas~\ref{lem:necessary:exact} and \ref{lem:necessary:exact:sample}}\label{appx:necessary}
This section restates and proves Lemmas~\ref{lem:necessary:exact} and \ref{lem:necessary:exact:sample}.

\subsection{Proof of Lemma~\ref{lem:necessary:exact}}\label{appx:necessary:exact}
\begin{lem:exact}[The necessity of adding states in the omniscient set]
    Adding a state from the omniscient set, $\xnew\in\fTrueSet$, is a necessary condition for \ac{RRTstar} to improve the current solution, $\ccur$,\squeezeWidowedWords
    \begin{equation*}
        \cnext < \ccur \implies \xnew \in \fTrueSet.
    \end{equation*}
    
    This condition is necessary but not \emph{sufficient} to improve the solution as the ability of states in $\fTrueSet$ to provide better solutions at any iteration depends on the structure of the tree (i.e., its optimality).
\end{lem:exact}
\begin{proof}
    At the end of iteration $\iter+1$, the cost of the best solution found by \ac{RRTstar} will be the minimum of the previous best solution, $\ccur$, and the best cost of any new or newly improved solutions, $\cnew$,
    \begin{equation}\label{eqn:thm:nec:next}
        \cnext = \min\set{\ccur,\cnew}.
    \end{equation}
    Each iteration of \ac{RRTstar} only adds connections to or from the newly added state, $\xnew$, and therefore all new or modified paths pass through this new state.
    The cost of any of these new paths that extend to the goal region will be bounded from below by the cost of the optimal solution of a path through $\xnew$,
    \begin{equation}\label{eqn:thm:nec:improve}
        \cnew \geq \fTrue{\xnew}.
    \end{equation}

    Lemma~\ref{lem:necessary:exact} is now proven by contradiction.
    Assume that \ac{RRTstar} has a solution with cost $\ccur$ after iteration $\iter$ and that it is improved at iteration $\iter+1$ by adding a state \emph{not} in the omniscient set, $\cnext < \ccur,\, \xnew\not\in\fTrueSet$.
    By \eqref{eqn:fset}, the costs of solutions through any $\xnew\not\in\fTrueSet$ are bounded from below by the current solution,
    \begin{equation*}
        \fTrue{\xnew} \geq \ccur,
    \end{equation*}%
    which by \eqref{eqn:thm:nec:improve} is also a bound on the cost of any new or modified solutions,
    \begin{equation*}
        \cnew \geq \fTrue{\xnew} \geq \ccur.
    \end{equation*}%
    By \eqref{eqn:thm:nec:next}, the cost of the best solution found by \ac{RRTstar} at the end of iteration $\iter+1$ must therefore be $\ccur$.
    This contradicts the assumption that the solution was improved by a state not in the omniscient set and proves Lemma~\ref{lem:necessary:exact}.
\end{proof}

\subsection{Proof of Lemma~\ref{lem:necessary:exact:sample}}\label{appx:necessary:sample}
\begin{lem:sample}[The necessity of sampling states in the omniscient set in holonomic planning]
    Sampling the omniscient set, $\xrand\in\fTrueSet$, is a necessary condition for \ac{RRTstar} to improve the current solution to a holonomic problem, $\ccur$, after an initial $\iterThresh$ iterations,
    \begin{equation*}
        \forall \iter \geq \iterThresh,\, \cnext < \ccur \implies \xrand \in \fTrueSet,
    \end{equation*}%
    for any sample distribution that maintains a nonzero probability over the entire omniscient set.
    
    For simplicity, this statement is limited to holonomic planning but it can be extended to specific constraints with appropriate assumptions.
\end{lem:sample}
\begin{proof}
    In \ac{RRT} (and therefore \ac{RRTstar}), the distribution of vertices in the graph approaches the sample distribution as the number of iterations approach infinity \cite{lavalle_tech98}.
    In the limit, all reachable regions of the problem domain with a nonzero sampling probability will therefore be sampled and the number of vertices in these regions will increase indefinitely with the number of iterations.
    This ever increasing number of vertices means that the worst-case distance between any state in a sampled subset and the nearest vertex in the graph will decrease indefinitely and monotonically.
    
    Lemma~\ref{lem:necessary:exact:sample} is now proven by contradiction.
    Assume that by iteration $\iterThresh$ there are a sufficient number and distribution of vertices in the tree such that all possible states in $\fTrueSet$ are no further than $\maxEdge$ from a vertex,
    \begin{equation}\label{eqn:cor:dist}
        \forall \statex\in\fTrueSet,\, \exists \statev\in\vertexSet \suchthat \norm{\statex-\statev}{2} < \maxEdge,
    \end{equation}
    and that \ac{RRTstar} has a solution with cost $\ccur$ after iteration $\iter\geq\iterThresh$.
    Now assume that \ac{RRTstar} improves the solution at iteration $\iter+1$ \emph{without} \emph{sampling} the omniscient set, $\cnext < \ccur,\, \xrand \not\in \fTrueSet$.

    As improving a solution requires \emph{adding} a state from the omniscient set, $\xnew\in\fTrueSet$, (Lemma~\ref{lem:necessary:exact}) this implies that the state added to the graph is not the randomly sampled state, $\xnew\not=\xrand$.
    These two states are related in holonomic planning by expansion constraints, \eqref{eqn:back:nearest} and \eqref{eqn:back:steer}, that find a new state as near as possible to $\xrand$ and no further than $\maxEdge$ from the nearest vertex in the tree.

    The triangle inequality implies that the nearest vertex to the sample, $\vnearest$, is also the nearest vertex to the proposed new state,
    \begin{equation*}
    \begin{aligned}
        \vnearest &\coloneqq \argmin_{\statev\in\vertexSet}\set{\norm{\xrand - \statev}{2}}\\
                  &\equiv \argmin_{\statev\in\vertexSet}\set{\norm{\xnew - \statev}{2}},
    \end{aligned}
    \end{equation*}
    which from \eqref{eqn:cor:dist} is bounded in its distance from $\xnew$ by
    \begin{equation}\label{eqn:cor:near_dist}
        \norm{\xnew - \vnearest}{2} < \maxEdge.
    \end{equation}
    Due to \eqref{eqn:back:steer}, the relationship in \eqref{eqn:cor:near_dist} is only satisfied in holonomic planning when $\xnew\equiv\xrand$.
    As by assumption the random sample is not a member of the omniscient set, $\xrand\not\in\fTrueSet$, then therefore neither is the newly added state, $\xnew\not\in\fTrueSet$, and by Lemma~\ref{lem:necessary:exact} the solution is not improved, $\cnext = \ccur$.
    This contradicts the assumption that the solution was improved by sampling a state not in the omniscient set and proves Lemma~\ref{lem:necessary:exact:sample}.
\end{proof}

\section{Proof of Lemma~\ref{lem:uniform}}\label{appx:uniform}
This section restates Lemma~\ref{lem:uniform} as presented by \cite{sun_fusion02} along with a full proof as presented in \cite{gammell_arxiv14b}.

\begin{lem:uniform}[The uniform distribution of samples transformed into a hyperellipsoid from a unit $\dimension$-ball. Originally Lemma 1 in \cite{sun_fusion02}]
    If the random points distributed in a hyperellipsoid are generated from the random points uniformly distributed in a hypersphere through a linear invertible nonorthogonal transformation, then the random points distributed in the hyperellipsoid are also uniformly distributed.
\end{lem:uniform}
\begin{proof}
    Let the sets $\ballSet\subset\real^\dimension$ and $\ellipseSet\subset\real^\dimension$ be the unit $\dimension$-dimensional ball and a $\dimension$-dimensional hyperellipsoid with radii $\set{\radius_{\counterj}}_{\counterj=1}^\dimension$, respectively, having measures of
    \begin{align*}
        \lebesgue{\ballSet} &= \unitBall{\dimension},\\
        \lebesgue{\ellipseSet} &= \unitBall{\dimension}\prod_{\counterj=1}^\dimension \radius_\counterj.
    \end{align*}
    Let $\ballPdf{\cdot}$ be the probability density function of samples drawn uniformly from the unit $\dimension$-ball such that,
    \begin{equation}\label{eqn:ballPdf}
        \ballPdf{\statex} \coloneqq
        \begin{cases}
            \dfrac{1}{\unitBall{\dimension}},& \forall\statex\in\ballSet\\
            0,              & \text{otherwise}.
        \end{cases}
    \end{equation}
    Let $\linTrans{\cdot}$ be an invertible transformation from the unit $\dimension$-ball to a hyperellipsoid such that,
    \begin{align*}
        \linTransSymb &: \; \ballSet \to \ellipseSet,\\
        \linTransSymb^{-1} &: \; \ellipseSet \to \ballSet.
    \end{align*}
    By definition, the probability density function in the hyperellipsoid, $\ellipsePdf{\cdot}$, resulting from applying this transformation to samples distributed in the unit $\dimension$-ball is then
    \begin{equation}\label{eqn:pdfDefn}
        \ellipsePdf{\statex} \coloneqq \ballPdf{\invLinTrans{\statex}} \left|\det\left( \left.\frac{d\invLinTransSymb}{d\xellipse}\right|_{\statex} \right) \right|.
    \end{equation}

    The proposed transformation in \eqref{eqn:transformDefn} has the inverse
    \begin{equation*}
        \invLinTrans{\xellipse} = \mathbf{L}^{-1}\left(\xellipse - \xcentre\right),
    \end{equation*}
    and the Jacobian
    \begin{equation}\label{eqn:jacobian}
        \frac{d\invLinTransSymb}{d\xellipse} = \mathbf{L}^{-1}.
    \end{equation}
    
    Substituting \eqref{eqn:jacobian} and \eqref{eqn:ballPdf} into \eqref{eqn:pdfDefn} gives,
    \begin{equation}\label{eqn:ellipsePdf}
        \ellipsePdf{\statex} \coloneqq 
        \begin{cases}
            \dfrac{1}{\unitBall{\dimension}}\left|\det\left( \mathbf{L}^{-1} \right) \right|,& \forall \statex \in \ellipseSet\\
            0,              & \text{otherwise},
        \end{cases}
    \end{equation}
    using the fact that $\invLinTrans{\statex} \in \ballSet \iff \statex \in \ellipseSet$.
    As $\ellipsePdf{\cdot}$ is constant for all $\xellipse \in \ellipseSet$, this proves that using \eqref{eqn:transformDefn} to transform uniformly distributed samples in the unit $\dimension$-ball results in a uniform distribution over the hyperellipsoid and proves Lemma~\ref{lem:uniform}.

    For hyperellipsoids whose axes are orthogonal (e.g., a prolate hyperspheroid), \eqref{eqn:ellipsePdf} can be expressed in a more familiar and intuitive form.
    Using \eqref{eqn:transformFinal} for $\linTrans{\cdot}$ and the orthogonality of rotation matrices makes \eqref{eqn:ellipsePdf}
    \begin{equation}\label{eqn:orthoPdfTemp}
        \ellipsePdf{\statex} \coloneqq 
        \begin{cases}
            \dfrac{1}{\unitBall{\dimension}}\left|\det\left( \mathbf{L'}^{-1}\ellipseRotation^T \right) \right|,& \forall \statex \in \ellipseSet\\
            0,              & \text{otherwise}.
        \end{cases}
    \end{equation}
    where $\mathbf{L}' = \diag\left( \radius_1, \radius_2, \ldots, \radius_n \right)$ is a diagonal matrix which then simplifies \eqref{eqn:orthoPdfTemp} to
    \begin{equation}\label{eqn:orthoPdf}
        \ellipsePdf{\statex} \coloneqq 
        \begin{cases}
            \dfrac{1}{\unitBall{\dimension}\prod_{\counterj=1}^\dimension \radius_\counterj},& \forall \statex \in \ellipseSet\\
            0,              & \text{otherwise},
        \end{cases}
    \end{equation}
    since the determinant is a linear operator, all rotation matrices have a unity determinant, $\det\left(\ellipseRotation\right) = 1$, and the determinant of a diagonal matrix is the product of its diagonal entries.
    As expected, \eqref{eqn:orthoPdf} is the inverse of the volume of an $\dimension$-dimensional hyperellipsoid with radii $\left\lbrace \radius_\counterj \right\rbrace_{\counterj=1}^\dimension$.
\end{proof}

\section{Proof of Lemma~\ref{lem:conv:expect}}\label{appx:infinite:expect}
This section restates and proves Lemma~\ref{lem:conv:expect}, which is used in support of Theorems~\ref{thm:conv:rrtstar}--\ref{thm:conv:informed}.
An earlier version of this proof appeared in \cite{gammell_arxiv14}.

\begin{lem:expect}[Expected next-iteration cost of minimum-path-length planning]
    The expected value of the next solution to a minimum-path-length problem, $\expect{\cnext}$, is bounded by
    \begin{equation*}
        \probBetter\frac{\dimension\ccur^2 + \cmin^2}{\left(\dimension + 1\right)\ccur} + \left(1-\probBetter\right)\ccur \leq \expect{\cnext} \leq \ccur,\revisit{eqn:lem:conv:expect}
    \end{equation*}
    where $\ccur$ is the current solution cost, $\cmin$ is the theoretical minimum solution cost, $\dimension$ is the state dimension of the planning problem, and $\probBetter = \prob{\xnew \in \fTrueSet}$ is the probability of adding a state that is a member of the omniscient set (i.e., that can belong to a better solution).
    While not explicitly shown, the subset, $\fTrueSet$, and the probability of improving the solution, $\probBetter$, are generally functions of the current solution cost.
    
    This lower bound is sharp over the set of all possible minimum-path-length planning problems and algorithm configurations and is exact for versions of \ac{RRTstar} with an infinite rewiring radius (i.e., $\maxEdge=\infty$, and $\rrrtstar=\infty$) searching an obstacle-free environment without constraints.
\end{lem:expect}%
\begin{proof}
    Proof of the upper bound is trivial.
    \ac{RRTstar} only accepts new solutions that improve its existing solution, assuring that the cost monotonically decreases,
    \begin{equation}\label{eqn:lem:infinite:expect:upper} 
        \cnext \leq \ccur.
    \end{equation}
    Proof of the lower bound comes from finding an exact expression for the expected value of the solution cost found in the absence of obstacles and constraints with an infinite rewiring neighbourhood.

    The expected solution cost of \ac{RRTstar} depends on the probability of sampling the omniscient set,
    \begin{align}\label{eqn:lem:infinite:expect:expectDefn}
        \expect{\cnext} ={} &\probBetter\expectst{\cnext}{\xnew \in\fTrueSet}\nonumber\\
         &{}+ \left(1-\probBetter\right)\expectst{\cnext}{\xnew \not\in\fTrueSet},
    \end{align}
    where $\probBetter = \prob{\xnew \in \fTrueSet}$.
    Adding a state from the omniscient set, $\fTrueSet$, is a necessary condition to improve the solution (Lemma~\ref{lem:necessary:exact}) and any other state will not change the solution cost, $\expectst{\cnext}{\xnew\not\in\fTrueSet} = \ccur$.
    This simplifies \eqref{eqn:lem:infinite:expect:expectDefn} to
    \begin{equation}\label{eqn:lem:infinite:expect:simpleDefn}
        \expect{\cnext} = \probBetter\expectst{\cnext}{\xnew \in\fTrueSet} + \left(1-\probBetter\right)\ccur.
    \end{equation}
    The costs of solutions found by adding states inside the omniscient are bounded from below by the optimal path through the newly added state,
    \begin{equation}\label{eqn:lem:infinite:expect:expect1}
        \expectst{\cnext}{\xnew\in\fTrueSet} \geq \expectst{\fTrue{\xnew}}{\xnew\in\fTrueSet},
    \end{equation}
    where $\fTrue{\statex}$ is the cost of the optimal path from the start to the goal constrained to pass through a state, $\statex$.
    With a uniform sample distribution over $\fTrueSet$ the right-hand side of \eqref{eqn:lem:infinite:expect:expect1} becomes 
    \begin{equation*}
        \expectst{\fTrue{\xnew}}{\xnew\in\fTrueSet} = \frac{1}{\lebesgue{\fTrueSet}}\int_{\fTrueSet} \fTrue{\xnew}dV.
    \end{equation*}
    
    When \ac{RRTstar} uses an infinite rewiring radius it attempts connections between every new state and the start and goal.
    In the absence of obstacles and constraints these paths will be feasible and represent the optimal solutions using the state.
    This makes the expected value of this best-case configuration of \ac{RRTstar} equivalent to the expected optimal solution cost in the absence of obstacles,
    \begin{equation}\label{eqn:lem:infinite:expect:expect2}
        \expectst{\cnext}{\xnew\in\fTrueSet}^* \equiv \expectst{\fTrue{\xnew}}{\xnew\in\fTrueSet}.
    \end{equation}
    The lower bound provided by \eqref{eqn:lem:infinite:expect:expect1} is therefore sharp over the set of all possible planning problems and algorithm configurations.
    
    In this absence of obstacles and constraints, the optimal solution using any state is given by \eqref{eqn:fBelow} and the omniscient set is the prolate hyperspheroid, $\fTrueSet\equiv\fBelowSet\equiv\phsSet$.
    The measure of the omniscient set, $\lebesgue{\fTrueSet}=\phsMeasure$, is given by \eqref{eqn:phsMeasure}.
    This allows \eqref{eqn:lem:infinite:expect:expect2} to be written as 
    \begin{align}\label{eqn:lem:infinite:expect:expect3}
        \expectst{\cnext}{\xnew\in\fTrueSet}^* = &\frac{1}{\phsMeasure}\int_{\phsSet}\left( \norm{\statex - \xstart}{2}
                                                   \vphantom{+ \norm{\xgoal - \statex}{2}}\right.\nonumber\\
                                                  &\qquad\quad\;\; \left.\vphantom{\norm{\statex - \xstart}{2}}
                                                  {}+ \norm{\xgoal - \statex}{2}\right)dV.
    \end{align}

    The prolate hyperspheroidal coordinates, $\mu, \nu, \psi_1, \ldots, \psi_{\dimension-2}$,
    \begin{align*}
        x_1 &= a\cosh\mu\cos\nu,\\
        x_2 &= a\sinh\mu\sin\nu\cos\psi_1,\nonumber\\
        x_3 &= a\sinh\mu\sin\nu\sin\psi_1\cos\psi_2,\nonumber\\
        &\vdots\nonumber\\
        x_{n-1} &= a\sinh\mu\sin\nu\sin\psi_1\sin\psi_2\ldots\sin\psi_{n-3}\cos\psi_{n-2},\nonumber\\
        x_{n} &= a\sinh\mu\sin\nu\sin\psi_1\sin\psi_2\ldots\sin\psi_{n-3}\sin\psi_{n-2},\nonumber
    \end{align*}
    and the parameterization $a = 0.5\cmin$, simplifies \eqref{eqn:fBelow} to
    \begin{equation}\label{eqn:lem:infinite:expect:elliptical}
        \fTrue{\statex} = \cmin\cosh{\mu}.
    \end{equation}

    Substituting \eqref{eqn:lem:infinite:expect:elliptical} and the prolate hyperspheroidal differential volume,
    \begin{align*}
        dV = a^\dimension\left(\sinh^2\mu + \sin^2\nu\right)& \sinh^{\dimension-2}\mu\sin^{\dimension-2}\nu\sin^{\dimension-3}\psi_1\ldots\\
        &\;\;\sin\psi_{\dimension-3}\,d\mu\,d\nu\,d\psi_1\,\ldots\,d\psi_{\dimension-2},
    \end{align*}
    into \eqref{eqn:lem:infinite:expect:expect3} results in
    \begin{align}\label{eqn:lem:infinite:expect:expect4}
        \expectSymb&\expectstBrackets{\cnext}{\xnew\in\fTrueSet}^* = \frac{\cmin^{\dimension+1}}{2^\dimension \phsMeasure}\int_{\mu = 0}^{\mu_\iter}\int_{\nu = 0}^{\pi}\int_{\psi_1 = 0}^{\pi}\nonumber\\
        &\qquad\ldots\int_{\psi_{\dimension-3} = 0}^{\pi}\int_{\psi_{\dimension-2} = 0}^{2\pi} \left(\sinh^2\mu + \sin^2\nu \right)\nonumber\\
        &\qquad\qquad\sinh^{\dimension-2}\mu\cosh\mu\sin^{\dimension-2}\nu\sin^{\dimension-3}\psi_1\nonumber\\
        &\qquad\qquad\ldots \sin\psi_{\dimension-3}\,d\mu\,d\nu\,d\psi_1,\ldots d\psi_{\dimension-2},
    \end{align}
    where the integration limit for $\mu$ is derived from \eqref{eqn:lem:infinite:expect:elliptical} as
    \begin{equation}\label{eqn:lem:infinite:expect:cosh_lim}
         \cosh\mu_\iter \coloneqq \frac{\ccur}{\cmin}.
     \end{equation}
    
    Integrating \eqref{eqn:lem:infinite:expect:expect4} requires applying a series of identities, first
    \begin{align*}
        \left(\dimension-1\right)\unitBall{\dimension-1} \equiv \int_{\psi_1 = 0}^{\pi} \ldots &  \int_{\psi_{\dimension-3} = 0}^{\pi} \int_{\psi_{\dimension-2} = 0}^{2\pi}\sin^{\dimension-3}\psi_1 \nonumber\\
        &\ldots \sin\psi_{\dimension-3}\,d\psi_1 \ldots \,d\psi_{\dimension-2},
    \end{align*}
    simplifies \eqref{eqn:lem:infinite:expect:expect4} to
    \begin{align}\label{eqn:lem:infinite:expect:expect5}
        \expectst{\cnext}{\xnew\in\fTrueSet}^*& = \frac{\left(\dimension-1\right)\cmin^{\dimension+1}\unitBall{\dimension-1}}{2^\dimension \phsMeasure}\nonumber\\
        &\int_{\mu = 0}^{\mu_i}\int_{\nu = 0}^{\pi}\left(\sinh^2\mu + \sin^2\nu \right)\nonumber\\
        &\sinh^{\dimension-2}\mu\cosh\mu\sin^{\dimension-2}\nu
        \,d\mu\,d\nu.
    \end{align}
    Next, the definite integral of the product of powers of $\sin$ and $\cos$,
    \begin{equation*}
        \int_0^\pi \sin^{2m-1}\theta \cos^{2n-1}\theta \,d\theta \equiv \betaFunc{m}{\dimension},
    \end{equation*}
    where $\betaFunc{\cdot}{\cdot}$ is the beta function,
    \begin{equation*}
        \betaFunc{m}{\dimension} \coloneqq \int_{0}^{1}t^{m-1}\left(1 - t\right)^{\dimension-1}\,dt,
    \end{equation*}
    is used to evaluate the integral over $\nu$ in \eqref{eqn:lem:infinite:expect:expect5}, giving
    \begin{align}\label{eqn:lem:infinite:expect:expect6}
        \expectSymb&\expectstBrackets{\cnext}{\xnew\in\fTrueSet}^* = \frac{\left(\dimension-1\right)\cmin^{\dimension+1}\unitBall{\dimension-1}}{2^\dimension \phsMeasure}\nonumber\\
        &\qquad\left(
            \betaFunc{\frac{\dimension-1}{2}}{\frac{1}{2}}
            \int_{\mu = 0}^{\mu_i}
            \sinh^\dimension\mu\cosh\mu
            \,d\mu\right.\nonumber\\
        &\qquad\left. {}+
            \betaFunc{\frac{\dimension+1}{2}}{\frac{1}{2}}
            \int_{\mu = 0}^{\mu_i}
            \sinh^{\dimension-2}\mu\cosh\mu
            \,d\mu
        \right).
    \end{align}
    The identity,
    \begin{equation*}
        \betaFunc{m+1}{\dimension} \equiv \frac{m}{m+\dimension}\betaFunc{m}{\dimension},
    \end{equation*}
    and the recursive nature of the $\dimension$-dimensional unit ball,
    \begin{equation*}
        \unitBall{\dimension} \equiv \betaFunc{\frac{\dimension+1}{2}}{\frac{1}{2}}\unitBall{\dimension-1},
    \end{equation*}
    simplifies \eqref{eqn:lem:infinite:expect:expect6} to
    \begin{align}\label{eqn:lem:infinite:expect:expect7}
        \expectSymb&\expectstBrackets{\cnext}{\xnew\in\fTrueSet}^* \hspace{-0.25ex}=\hspace{-0.25ex} \frac{\cmin^{\dimension+1}\unitBall{\dimension}}{2^\dimension \phsMeasure}\nonumber
        \left(
            \dimension
            \int_{\mu = 0}^{\mu_i}
            \hspace{-0.5ex}\sinh^\dimension\mu\cosh\mu
            \,d\mu\right.\nonumber\\
        &\quad\qquad\qquad\left. {}+
            \left(\dimension-1\right)
            \int_{\mu = 0}^{\mu_i}
            \sinh^{\dimension-2}\mu\cosh\mu
            \,d\mu
        \right).
    \end{align}
    The indefinite integral,
    \begin{equation*}
        \int\sinh^m\theta\cosh\theta\,d\theta \equiv \frac{\sinh^{m+1}\theta}{m+1},
    \end{equation*}
    is then used to evaluate \eqref{eqn:lem:infinite:expect:expect7}, giving
    \begin{align}\label{eqn:lem:infinite:expect:expect8}
        \expectSymb&\expectstBrackets{\cnext}{\xnew\in\fTrueSet}^* = \frac{\cmin^{\dimension+1}\unitBall{\dimension}}{2^\dimension \phsMeasure}\nonumber\\
        &\qquad\qquad\qquad
        \left(
            \frac{\dimension}{\dimension+1}\sinh^{\dimension+1}\mu_i
        +
            \sinh^{\dimension-1}\mu_i
        \right).
    \end{align}
    Using \eqref{eqn:phsMeasure} to expand the measure $\phsMeasure$ in \eqref{eqn:lem:infinite:expect:expect8} cancels the measure of the unit $\dimension$-ball, giving 
    \begin{align}\label{eqn:lem:infinite:expect:expect9}
        \expectSymb&\expectstBrackets{\cnext}{\xnew\in\fTrueSet}^* = \frac{\cmin^{\dimension+1}}{\ccur\left(\ccur^2 - \cmin^2\right)^{\frac{\dimension - 1}{2}}}\nonumber\\
        &\qquad\qquad\qquad
        \left(
            \frac{\dimension}{\dimension+1}\sinh^{\dimension+1}\mu_i
        +
            \sinh^{\dimension-1}\mu_i
        \right).
    \end{align}
    Using the relationship
    \begin{equation*}
        \cosh{\mu} = b \iff \sinh{\mu} = \sqrt{b^2 - 1},
    \end{equation*}
    some algebraic manipulation, and \eqref{eqn:lem:infinite:expect:cosh_lim} finally simplifies \eqref{eqn:lem:infinite:expect:expect9} to
    \begin{equation*}
        \expectst{\cnext}{\xnew\in\fTrueSet}^* = \frac{\dimension\ccur^2 + \cmin^2}{\left(\dimension+1\right)\ccur},
    \end{equation*}
    an exact value for the best-case expected solution cost of \ac{RRTstar}.
    This result allows \eqref{eqn:lem:infinite:expect:simpleDefn} to be written as the sharp bound,
    \begin{equation*}
        \expect{\cnext} \geq \probBetter\frac{\dimension\ccur^2 + \cmin^2}{\left(\dimension + 1\right)\ccur} + \left(1-\probBetter\right)\ccur,
    \end{equation*}
    which when combined with \eqref{eqn:lem:infinite:expect:upper} proves Lemma~\ref{lem:conv:expect}.
\end{proof}

\section{Proofs of Theorems~\ref{thm:conv:rrtstar}--\ref{thm:conv:informed}}\label{appx:infinite:rate}
This section restates and proves Theorems~\ref{thm:conv:rrtstar}--\ref{thm:conv:informed}.

\subsection{Proof of Theorem~\ref{thm:conv:rrtstar}}\label{appx:infinite:rate:rrtstar}
\begin{thm:conv_rrtstar}[Sublinear convergence of \acs{RRTstar} in holonomic minimum-path-length planning]
    \ac{RRTstar} converges \emph{sublinearly} towards the optimum of holonomic minimum-path-length planning problems,
    \begin{equation*}
        \expect{\rrtstarRate} = 1.\revisit{eqn:thm:conv:rrtstar}
    \end{equation*}
    
    For simplicity, this statement is limited to holonomic planning but it can be extended to specific constraints by expanding Lemma~\ref{lem:necessary:exact:sample}.
\end{thm:conv_rrtstar}
\begin{proof}
    The expected rate of convergence (Definition~\ref{defn:conv}) of \ac{RRTstar} is
    \begin{equation}\label{eqn:thm:conv:rrtstar:defn1}
        \expect{\rrtstarRate} = \expect{\limitoinf\frac{\ccur - \copt}{\cprev - \copt}},
    \end{equation}
    since $\forall\iter,\, \ccur\geq\copt$.
    As \ac{RRTstar} almost-surely converges asymptotically to the optimum, this sequence also almost-surely converges to a finite value, $0 \leq \rrtstarRate \leq 1$,
    \begin{equation*}
        \prob{\limitoinf \frac{\ccur - \copt}{\cprev - \copt} = \rrtstarRate} = 1.
    \end{equation*}
    By Lebesgue's dominated convergence theorem this allows the expectation operator to be brought inside the limit of \eqref{eqn:thm:conv:rrtstar:defn1}, giving
    \begin{equation}\label{eqn:thm:conv:rrtstar:defn2}
        \expect{\rrtstarRate} = \limitoinf\frac{\expect{\ccur} - \copt}{\cprev - \copt},
    \end{equation}
    since $\ccur$ is the only random variable at iteration $\iter$.

    Lemma~\ref{lem:conv:expect} provides sharp bounds for the expected solution cost at any iteration, $\expect{\ccur}$, with the lower-bound corresponding to an infinite rewiring radius in the absence of obstacles and constraints.
    Substituting this lower bound and that $\copt=\cmin$ in the absence of obstacles into \eqref{eqn:thm:conv:rrtstar:defn2} and simplifying gives an expression for the expected \emph{best-case} convergence rate,
    \begin{equation*}
        \expect{\rrtstarRate^*} = 1 +  \frac{1}{\left(\dimension+1\right)}
                                        \limitoinf \frac{
                                                            \probBetter\left(\cmin^2 - \cprev^2\right)
                                                        }
                                                        {
                                                            \left(\cprev^2 - \cmin\cprev\right)
                                                        },
    \end{equation*}
    such that $\expect{\rrtstarRate^*} \leq \expect{\rrtstarRate}$ is a sharp bound over all possible planning problems and algorithm configurations.
    Applying l'H\^{o}pital's rule \cite{hopitals_rule} with respect to $\cprev$ gives
    \begin{align}\label{eqn:thm:conv:rrtstar:lhopital}
        \expect{\rrtstarRate^*} = 1 &{}+ \frac{1}{\left(\dimension+1\right)}\nonumber\\
                                        &\quad
                                        \limitoinf \frac{
                                                             \frac{\partial\probBetter}{\partial\cprev}\left(\cmin^2 - \cprev^2\right)
                                                               - 2\probBetter\cprev
                                                        }
                                                        {
                                                           \left(2\cprev-\cmin\right)
                                                        }.
    \end{align}
    
    As iterations go to infinity the probability of adding a sample in $\fTrueSet$ becomes the probability of sampling it (Lemma~\ref{lem:necessary:exact:sample}).
    The lower bound from Lemma~\ref{lem:conv:expect} is for an obstacle-{} and constraint-free problem and therefore the informed set is the omniscient set, $\fTrueSet \equiv \fBelowSet$, and the probability of sampling it is given by \eqref{eqn:sampleProb} with a partial derivative of
    \begin{align*}
        \frac{\partial\probBetter}{\partial\cprev} = \frac{\pi^{\frac{\dimension}{2}}}{2^\dimension\gammaFunc{\frac{\dimension}{2} + 1}\lebesgue{\samplingSet}}
                                                              &\left(\cprev^2 - \cmin^2\right)^{\frac{\dimension-1}{2}}\\
                                                              &\left(1 + \frac{\left(\dimension-1\right)\cprev^2}{\left(\cprev^2-\cmin^2\right)} \right).
    \end{align*}
    
    Almost-sure convergence to $\cmin$ implies $\limitoinf\cprev=\cmin$ and therefore $\limitoinf\probBetter=0$ and $\limitoinf\frac{\partial\probBetter}{\partial\cprev}=0$, making \eqref{eqn:thm:conv:rrtstar:lhopital},
    \begin{equation*}
        \expect{\rrtstarRate^*} = 1.
    \end{equation*}
    As by definition the expected rate of convergence of \ac{RRTstar} is bounded by,
    \begin{equation*}
        \expect{\rrtstarRate^*} \leq \expect{\rrtstarRate} \leq 1,
    \end{equation*}
    this result proves Theorem~\ref{thm:conv:rrtstar}.
\end{proof}

\subsection{Proof of Theorem~\ref{thm:conv:reject}}\label{appx:infinite:rate:reject}
\begin{thm:conv_reject}[Linear convergence of \acs{RRTstar} with adaptive rectangular rejection sampling in holonomic minimum-path-length planning]
    \ac{RRTstar} with adaptive rectangular rejection sampling converges at best \emph{linearly} towards the optimum of holonomic minimum-path-length planning problems but factorially approaches sublinear convergence with increasing state dimension,
    \begin{equation*}
        1 - \frac{\pi^{\frac{\dimension}{2}}}{\left(\dimension+1\right)2^{\dimension-1}\gammaFunc{\frac{\dimension}{2}+1}} \leq \expect{\rejectRate} \leq 1. \revisit{eqn:thm:conv:reject}
    \end{equation*}
    
    For simplicity, this statement is limited to holonomic planning but it can be extended to specific constraints by expanding Lemma~\ref{lem:necessary:exact:sample}.
\end{thm:conv_reject}
\begin{proof}
    Proof of Theorem~\ref{thm:conv:reject} follows that of Theorem~\ref{thm:conv:rrtstar} but with the probability of adding a new state from $\fTrueSet$ instead calculated from \eqref{eqn:sampleProb} using \eqref{eqn:thm:curse:tightMeasure}, as
    \begin{equation}\label{eqn:thm:conv:reject:pdf}
        \pdfSymb_{\solutionCost} \leq \frac{\pi^{\frac{\dimension}{2}}}{2^\dimension\gammaFunc{\frac{\dimension}{2} + 1}}.
    \end{equation}
    As $\frac{\partial\probBetter}{\partial\cprev}=0$, \eqref{eqn:thm:conv:rrtstar:lhopital} becomes
    \begin{equation}\label{eqn:thm:conv:reject:lhopital}
        \expect{\rejectRate^*}  = 1 - \frac{1}{\left(\dimension+1\right)}
                                        \limitoinf\frac{
                                                          2\probBetter\cprev
                                                       }
                                                       {
                                                          \left(2\cprev-\cmin\right)
                                                       }.
    \end{equation}
    Noting that almost-sure convergence to $\cmin$ implies $\limitoinf\cprev=\cmin$ and substituting \eqref{eqn:thm:conv:reject:pdf} into \eqref{eqn:thm:conv:reject:lhopital} results in
    \begin{align*}
        \expect{\rejectRate^*} &= 1 - \frac{2\probBetter}{\left(\dimension+1\right)},\\
                             & \geq 1 - \frac{\pi^{\frac{\dimension}{2}}}{\left(\dimension+1\right)2^{\dimension-1}\gammaFunc{\frac{\dimension}{2} + 1}}.
    \end{align*}
    As by definition the expected rate of convergence of \ac{RRTstar} with rectangular rejection sampling is bounded by,
    \begin{equation*}
        \expect{\rejectRate^*} \leq \expect{\rejectRate} \leq 1.
    \end{equation*}
    This result proves Theorem~\ref{thm:conv:reject} with sharp bounds over all possible holonomic planning problems and algorithm configurations.
\end{proof}

\subsection{Proof of Theorem~\ref{thm:conv:informed}}\label{appx:infinite:rate:informed}
\begin{thm:conv_informed}[Linear convergence of Informed \acs{RRTstar} in holonomic minimum-path-length planning]
    Informed \acs{RRTstar} converges at best \emph{linearly} towards the optimum of holonomic minimum-path-length planning problems,
    \begin{equation*}
         \frac{n-1}{n+1} \leq \expect{\informedRate} \leq 1,\revisit{eqn:thm:conv:informed}
    \end{equation*}
    where the lower-bound occurs exactly with an infinite rewiring neighbourhood in the absence of obstacles and constraints.
    
    For simplicity, this statement is limited to holonomic planning but it can be extended to specific constraints by expanding Lemma~\ref{lem:necessary:exact:sample}.
\end{thm:conv_informed}
\begin{proof}
    Proof of Theorem~\ref{thm:conv:informed} follows that of Theorem~\ref{thm:conv:rrtstar} but with a unity probability of adding a new state from $\fTrueSet$.
    From \eqref{eqn:thm:conv:rrtstar:lhopital}, the convergence rate of Informed \acs{RRTstar} is then,
    \begin{equation*}
        \expect{\informedRate^*} = 1 - \frac{1}{\left(\dimension+1\right)}
                                        \limitoinf\frac{
                                                            2\cprev
                                                       }
                                                       {
                                                            \left(2\cprev-\cmin\right)
                                                       }.
    \end{equation*}
    As almost-sure convergence to $\cmin$ implies $\limitoinf\cprev = \cmin$, this gives,
    \begin{equation*}
        \expect{\informedRate^*} = \frac{\dimension-1}{\dimension+1}.
    \end{equation*}
    As by definition the expected rate of convergence of \ac{RRTstar} with rectangular rejection sampling is bounded by,
    \begin{equation*}
        \expect{\informedRate^*} \leq \expect{\informedRate} \leq 1.
    \end{equation*}
    This result proves Theorem~\ref{thm:conv:informed} with sharp bounds over all possible holonomic planning problems and algorithm configurations.
\end{proof}

\section*{Acknowledgment}
The authors would like to thank the editorial board and reviewers for their helpful comments.
We would also like to thank Laszlo-Peter Berczi for discussions on precision and recall, Christopher Dellin and Michael Koval for discussions and insight on the sampling of arbitrarily overlapping shapes, Jennifer King and Clinton Liddick for help running the \acs{HERB} experiments, and Paul Newman for providing the time to finish this manuscript at the University of Oxford.

%
%
%
%
%
%
%


%
%

%
\begin{IEEEbiography}[{\includegraphics[width=1in,height=1.25in,clip,keepaspectratio]{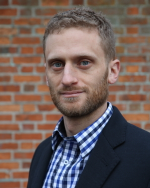}}]{Jonathan D.\ Gammell} is a Departmental Lecturer in Robotics with the Oxford Robotics Institute (ORI) at the University of Oxford, Oxford, United Kingdom.
He performed this work at the Autonomous Space Robotics Lab at the University of Toronto, Toronto, Canada during his Ph.\ D.\ degree.
He is interested in developing conceptually sound approaches to the fundamental problems of real-world autonomous robotics.
\end{IEEEbiography}

\begin{IEEEbiography}[{\includegraphics[width=1in,height=1.25in,clip,keepaspectratio]{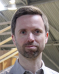}}]{Timothy D.\ Barfoot} is a Professor at the University of Toronto Institute for Aerospace Studies (UTIAS). He holds the Canada Research Chair (Tier II) in Autonomous Space Robotics and works in the area of guidance, navigation, and control of mobile robots in a variety of applications.
He is interested in developing methods to allow mobile robots to operate over long periods of time in large-scale, unstructured, three-dimensional environments, using rich onboard sensing (e.g., cameras and laser rangefinders) and computation.

Dr. Barfoot took up his position at UTIAS in May 2007, after spending four years at MDA Space Missions, where he developed autonomous vehicle navigation technologies for both planetary rovers and terrestrial applications such as underground mining.
He sits on the editorial boards of the International Journal of Robotics Research and the Journal of Field Robotics.
He recently served as the General Chair of Field and Service Robotics (FSR) 2015, which was held in Toronto.
\end{IEEEbiography}
\begin{IEEEbiography}[{\includegraphics[width=1in,height=1.25in,clip,keepaspectratio]{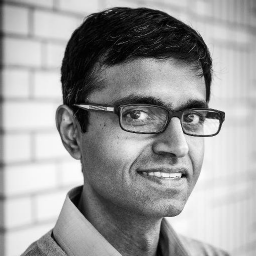}}]{Siddhartha S.\ Srinivasa} is the Boeing Endowed Professor at the School of Computer Science and Engineering, University of Washington. He works on robotic manipulation, with the goal of enabling robots to perform complex manipulation tasks under uncertainty and clutter, with and around people. To this end, he founded the Personal Robotics Lab in 2005. Dr.\ Srinivasa is also passionate about building end-to-end systems (HERB, ADA, HRP3, CHIMP, Andy, among others) that integrate perception, planning, and control in the real world. Understanding the interplay between system components has helped produce state of the art algorithms for robotic manipulation, motion planning, object recognition and pose estimation (MOPED), and dense 3D modelling (CHISEL, now used by Google Project Tango), and mathematical models for Human-Robot Collaboration.
\end{IEEEbiography}
\end{document}